\documentclass{article}
\usepackage{natbib}
\usepackage{arxiv,times}
\usepackage{enumitem}
\usepackage{subcaption}
\usepackage{microtype}
\usepackage{graphicx}
\usepackage{booktabs} 
\usepackage{colortbl}
\usepackage{pifont}
\usepackage{hyperref}
\usepackage{algorithm}
\usepackage{algorithmic}
\usepackage{setspace}

\usepackage{multirow}

\usepackage{amsmath,mathrsfs,mathtools,amsfonts,amssymb,amsthm}
\usepackage[capitalize,noabbrev]{cleveref}

\allowdisplaybreaks
\theoremstyle{plain}
\newtheorem{theorem}{Theorem}[section]
\newtheorem{proposition}[theorem]{Proposition}
\newtheorem{lemma}[theorem]{Lemma}

\theoremstyle{definition}
\newtheorem{definition}[theorem]{Definition}

\theoremstyle{remark}
\newtheorem{remark}{Remark}

\usepackage[textsize=tiny]{todonotes}
\usepackage{comment}
\mathchardef\mhyphen="2D
\DeclareMathOperator*{\argmax}{argmax}

\newcommand\given[1][]{\:#1\vert\:}

\usepackage[utf8]{inputenc} 
\usepackage[T1]{fontenc}    
\usepackage{hyperref}       
\usepackage{url}            
\usepackage{booktabs}       
\usepackage{amsfonts}       
\usepackage{nicefrac}       
\usepackage{microtype}      
\usepackage{xcolor}         
\usepackage{enumitem}
\usepackage{lscape}
\usepackage{tabularx}
\newcommand*{\email}[1]{\href{mailto:#1}{#1}}

\newtheorem*{rep@theorem}{\rep@title}
\newcommand{\newreptheorem}[2]{%
\newenvironment{rep#1}[1]{%
 \def\rep@title{\textbf{#2 \ref{##1}}}%
 \begin{rep@theorem}}%
 {\end{rep@theorem}}}
\makeatother

\newreptheorem{theorem}{Theorem}
\newreptheorem{lemma}{Lemma}

\usepackage{authblk}
\allowdisplaybreaks

\title{Optimal Multi-Objective Best Arm Identification with Fixed Confidence}

\renewcommand{\shorttitle}{\small Optimal Multi-Objective Best Arm Identification with Fixed Confidence}

\author[1]{Zhirui Chen}
\author[2]{P. N. Karthik}
\author[1]{Yeow Meng Chee}
\author[1]{Vincent Y. F. Tan}
\affil[1]{National University of Singapore
\thanks{Zhirui Chen is with the Department of Electrical and Computer Engineering and the Department of Industrial Systems Engineering and Management, National University of Singapore. Email: \email{zhiruichen@u.nus.edu}.

Yeow Meng Chee is with the Department of Industrial Systems Engineering and Management, National University of Singapore. 
Email: \email{ymchee@nus.edu.sg}.

Vincent~Y.~F.~Tan is with the Department of Mathematics and the Department of Electrical and Computer Engineering, National University of Singapore. 
Email: \email{vtan@nus.edu.sg}. 
}}

\affil[2]{Indian Institute of Technology, Hyderabad 
\thanks{P.~N.~Karthik is with the Department of Artificial Intelligence at the Indian Institute of Technology, Hyderabad.

Email: \email{pnkarthik@ai.iith.ac.in}.
}}

\begin{document}

\maketitle

\begin{abstract}
    We consider a multi-armed bandit setting with finitely many arms, in which each arm yields an $M$-dimensional vector reward upon selection. We assume that the reward of each dimension (a.k.a. {\em objective}) is generated independently of the others. The best arm of any given objective is the arm with the largest component of mean corresponding to the objective. The end goal is to identify the best arm of {\em every} objective in the shortest (expected) time subject to an upper bound on the probability of error (i.e., fixed-confidence regime). We establish a problem-dependent lower bound on the limiting growth rate of the expected stopping time, in the limit of vanishing error probabilities. This lower bound, we show, is characterised by a max-min optimisation problem that is computationally expensive to solve at each time step. 
    We propose an algorithm that uses the novel idea of {\em surrogate proportions} to sample the arms at each time step, eliminating the need to solve the max-min optimisation problem at each step. We demonstrate theoretically that our algorithm is asymptotically optimal. In addition, we provide extensive empirical studies to substantiate the efficiency of our algorithm.  While existing works on pure exploration with multi-objective multi-armed bandits predominantly focus on {\em Pareto frontier identification}, our work fills the gap in the literature by conducting a formal investigation of the multi-objective best arm identification problem.
\end{abstract}

\section{Introduction}
\label{sec:intro}

Multi-armed bandit (MAB)~\citep{thompson1933likelihood} is a sequential decision-making paradigm where an agent sequentially pulls one out of $K$ finitely many arms and receives a corresponding reward at each time step, with widespread applications in clinical trials, internet advertising, and recommender systems \citep{lattimore_szepesvari_2020}. In the classical MAB setup, the rewards from the arms are independent and identically distributed (i.i.d.), and real-valued (one-dimensional). In contrast, the {\em multi-objective multi-armed bandit} (MO-MAB) setup proposed by~\citet{drugan2013designing} allows for i.i.d.\ multi-dimensional (vector) rewards from the arms, with the reward of any given dimension (a.k.a the {\em objective}) being independent of the others or, more generally, a function of the rewards of the others. Defining the best arm of an objective as the arm with the largest mean reward corresponding to the objective, it is evident that in the MO-MAB setup, distinct objectives may possess distinct best arms, leading to the possibility of an arm being optimal for one objective and sub-optimal for another, thereby amplifying the complexity of identifying one or more best arms. In this paper, we study the problem of recovering the best arm of {\em every} objective in the shortest (expected) time, while ensuring that the probability of error is within a prescribed threshold ({\em fixed-confidence} regime).

\textbf{Motivation \ } \ 
{\color{black} Consider the task of deploying advertisements from a candidate set of advertisements, on platforms such as YouTube and Twitch. Here, selecting an advertisement to launch on any given day is analogous to pulling an arm. The feedback obtained from deploying a specific advertisement is inherently multi-dimensional, comprising various video-specific metrics such as user engagement and view rates, as well as demography-specific metrics such as viewer age and gender. The task of identifying the optimal advertisement for different demographic segments, which is crucial for maximizing revenue, translates to finding the best arm for each objective (e.g., each age group).}

As such, the problem of best arm identification (BAI) poses a non-trivial challenge, primarily owing to the inherent uncertainty associated with the true reward distribution of each arm. This challenge is further exacerbated when rewards are multi-dimensional (as in the MO-MAB setting). As delineated in prior works, numerous practical applications exhibit rewards that are multi-dimensional in nature, as opposed to being solely scalar, such as hardware design~\citep{zuluaga2016}, drug development and dose identification~\citep{lizotte2016multi} in clinical trials, and electric battery control~\citep{busa2017multi}. However, the existing works on pure exploration in MO-MAB settings are mainly focused on Pareto frontier identification. The identification of the best arm for each objective, an inherent task in MO-MAB scenarios, has not received comprehensive scholarly attention. This study seeks to fill the research gap in this domain.

\subsection{Overview of Existing Works}

Multi-objective bandits and fixed-confidence BAI have both been extensively investigated in the literature. 
This section highlights a collection of recent studies addressing these topics. For the problem of BAI in the fixed-confidence regime, \citet{garivier2016optimal} first proposed the well-known {\sc Track-and-Stop (TaS)} algorithm with two variants (C-Tracking and D-tracking), and demonstrated the optimality of these variants in the asymptotic limit of vanishing error probabilities. The basic premise for achieving asymptotic optimality, they showed, is to pull arms according to an {\em oracle weight} that is derived from the problem instance-dependent lower bound. Later, \citet{degenne2020gamification} and \citet{jedra2020optimal} specialised the {\sc TaS} algorithm to the linear bandit setting, while still maintaining asymptotic optimality. For more general structured bandits with non-linear structural dependence between the mean rewards of the arms, \citet{wang2021fast} proposed an efficient {\sc TaS}-type algorithm and a novel lower bound for the structured MAB setting, and further established the asymptotic optimality of their algorithm. \citet{mukherjee2023best} also proposed an efficient scheme for achieving asymptotic optimality without solving for the oracle weight at each time step.
It is noteworthy that the aforementioned studies deal with a {\em single objective}, whereas our research deals more generally with {\em multiple objectives}.

\citet{degenne2019pure} explored the problem of fixed-confidence BAI with {\em multiple correct answers} (i.e., multiple best arms), with the objective of identifying any one of the correct answers. They proposed the {\sc Sticky Track-and-Stop} algorithm along the lines of C-Tracking and demonstrated its asymptotic optimality. While their setup appears to bear similarities with ours, it is worth noting that their work aims to identify {\em one} among several correct answers, whereas our study focuses on uncovering {\em all} correct answers (i.e., the best arm of {\em every} objective).

\citet{drugan2013designing} introduced the MO-MAB setting as well as two associated metrics---the {\em Pareto regret} and the {\em scalarized regret}. The authors  proposed two UCB-like algorithms to optimize these two metrics. Subsequently, several factions of researchers have predominantly concentrated on the study of Pareto regret.~\citet{turgay2018multi} tackled the problem of Pareto regret minimization by introducing a similarity assumption regarding the means of arms, and designed an algorithm that achieves a regret upper bound of the order $\tilde{O}\left(T^{\left(1+d_p\right) /\left(2+d_p\right)}\right)$, where $d_p$ is the number of dimensions that is a function of the arm vectors and the environmental context.~\citet{xu2023pareto} defined the notion of Pareto regret in the context of adversarial bandits~\citep[Chapter 11]{lattimore_szepesvari_2020}, and designed an algorithm achieving near-optimality up to a factor of $\log T$ in both adversarial and stochastic bandit environments.

On the topic of pure exploration in the MO-MAB setting, a popular line of work is {\em Pareto optimal arm identification} or {\em Pareto frontier identification}. Considering specifically the fixed-confidence regime, \citet{auer2016pareto} proposed a successive elimination (SE)-type algorithm to identify all the Pareto optimal arms. For any given confidence level, the upper bound on the sample complexity of their algorithm matches their lower bound up to a logarithmic factor that is a function of 
the sub-optimality gaps of the arms. Along similar lines, \citet{ararat2023vector} present another SE-type algorithm to identify all  Pareto $(\epsilon,\delta)$-PAC arms, a generalization of Pareto optimal arms. More recently, \citet{kim2023pareto} developed a framework for analysing the problem of Pareto frontier identification in linear bandits. Their proposed algorithm is nearly optimal up to a logarithmic factor involving the minimum of 
the arm sub-optimality gaps and the algorithm's accuracy parameter. {\color{black} While the best arm of each objective is notably also Pareto optimal, our work goes beyond merely identifying a subset of Pareto optimal arms, hence significantly advancing the state-of-the-art in multi-objective bandits. Existing works on Pareto optimal arm identification in the MO-MAB setting lack the capability to identify the best arm of each objective, a task that our work accomplishes}. We discuss in Appendix~\ref{sec:differences_with_pfi} further details of the significant differences between our setting and Pareto optimal arm identification.

In the realm of pure exploration with vector-valued payoffs, the work of \citet{prabhu2022sequential} studies the problem of multi-hypothesis testing with vector arm rewards under the fixed-confidence regime. The authors provide insightful contributions by establishing both lower and upper bounds on the expected stopping time. Notably, their generic problem formulation encapsulates both single-objective best arm identification (BAI) and its multi-objective counterpart.
However, when specialized to multi-objective BAI, a computational bottleneck arises in their proposed algorithm due to the emergence of a multi-dimensional optimization routine akin to the core optimization procedure in {\sc TaS}. This introduces computational inefficiencies, warranting computationally more efficient solutions for high-dimensional problems.
In a parallel vein, the work of \citet{shang2020stochastic} introduces the concept of the ``relative vector loss,'' tied to the sup-norm of vector rewards. Their primary objective lies in identifying the best arm (defined in terms of vector losses and hence vastly different from ours) within the fixed-confidence regime. While their focus is on identifying the overall best arm, our framework necessitates identifying the best arm of each objective.

The MO-MAB setup (with $M$ independent objectives) appears to bear connections to the federated multi-armed bandit setting with a single server and $M$ independent clients, with each client having access to all arms or a subset thereof and striving to determine its best arm, as studied for instance in the recent works of \citet{kota2023almost,shi2021federated,chen2022federated}. Despite the apparent high-level similarities between the two settings, there exists one crucial distinction. In the federated setting, each client may choose an arm of its choice independently of the other clients, thereby leading to the possibility of accruing $M$ rewards from $M$ distinct arms at any given time instant. However, in the MO-MAB setup, an $M$-dimensional reward is generated from a single arm at each time instant.  In essence, the MO-MAB setup is equivalent to a federated learning setup in which every client has access to all the arms, and all clients are compelled to pull the same arm at each time instant. 

\subsection{Our Contributions}

While existing works on pure exploration in multi-objective bandits primarily focus on Pareto frontier identification, we bridge the gap by investigating the problem of identifying the best arm of each objective under the fixed-confidence regime. 

We provide an asymptotic lower bound for the problem and complement it with an algorithm that achieves the lower bound up to a multiplicative constant of $1+\eta$, where $\eta>0$ is a tuneable parameter that can be chosen arbitrarily close to $0$. We show that the lower bound is characterised by the solution to a max-min optimisation problem that is reminiscent of fixed-confidence BAI problems. The basic premise upon which asymptotically optimal algorithms of the prior works such as D-Tracking \citep{garivier2016optimal} and Sticky {\sc TaS} \citep{degenne2019pure} operate is to compute the max-min optimisation at each time step to evaluate the oracle weight at each time instant, and to pull arms according to the (empirical) oracle weight in order to guarantee asymptotic optimality. {\color{black} However, these approaches may be inefficient, as the oracle weight, to the best of our knowledge, does not have a closed-form solution in multi-objective cases. }
Instead of following the (empirical) oracle weight, we propose a novel technique to sample the arms at each step, based on the idea of {\em surrogate proportions}. These surrogate proportions serve as proxy for the oracle weight and can be computed efficiently.
\section{Preliminaries and Problem Setup}
\label{sec:problem-setup}
\vspace{-.1in}
Let $\mathbb{N}$ denote the set of positive integers. For $n \in \mathbb{N}$, let $[n] \coloneqq \{1, \ldots, n\}$. We consider a multi-armed bandit with $K$ arms in which each arm is associated with $M$ independent {\em objectives}. Pulling arm $A_t \in [K]$ at time step $t$ yields an $M$-dimensional reward $\mathbf{r}_t=[r_{t,m}: m \in [M]]^\top \in \mathbb{R}^M$, where $r_{t,m}=\mu_{A_t,m}+ \eta_{t,m}$; here, $\mu_{A_t,m}\in \mathbb{R}$ is the unknown mean corresponding to objective $m$ of arm $A_t$, and $\eta_{t,m}$ is an independent standard normal random variable. Let $v=[\mu_{i,m}: (i, m) \in [K] \times [M]]^\top$ denote a {\em problem instance} in which $\mu_{i,m}$ is the mean corresponding to objective $m$ of arm $i$. Arm $i$ is said to be the best arm of objective $m$ if it has the highest mean in dimension $m$ across all arms, i.e., $\mu_{i,m} > \mu_{j,m}$ for all $j \neq i$.
Without loss of generality, we assume that each objective has a unique best arm, and we write $\mathcal{P}$ to denote the set of all problem instances with a unique best arm for each objective. We write $I^*(v) = (i^*_1(v), i^*_2(v), \dots, i^*_M(v))$ to denote the collection of best arms under instance $v$; here, $i^*_m(v)$ is the best arm of objective $m$.

Given an error probability threshold $\delta \in (0,1)$, the goal is to identify the set of best arms $I^*(v)$ in the shortest time, while ensuring that the error probability is within $\delta$. Formally, an {\em algorithm} (or {\em policy}) for identifying the best arms is a tuple $\pi =(A, \tau, \widehat{I})$ consisting of the following components.
\begin{itemize}
    \item An {\em arms selection rule} $A=\{A_t\}_{t=1}^{\infty}$ for pulling the arms at each time instant. Here, $A_t=A_t(A_{1:t-1}, \mathbf{r}_{1:t-1},\delta)$ is a (random) function that takes input as the history of   all the arms pulled and rewards obtained up to time step $t-1$ as well as the confidence $\delta$, and outputs the arm to be pulled at time step $t$.

    \item A {\em stopping rule} that dictates the stopping time $\tau$ at which to stop further selection of arms.

    \item A {\em recommendation} $\widehat{I} \in [K]^M$ of best arm estimates at the stopping time $\tau$.
\end{itemize}
For simplicity, let $\tau_\delta$ and $\widehat{I}_\delta$ denote the stopping time and recommendation under confidence $\delta$, respectively.
Our interest is in the class of all {\em $\delta$-PAC policies}, defined as 
\begin{align}
    &\Pi(\delta) \coloneqq \{\pi: \ \mathbb{P}_v^\pi(\tau_\delta < +\infty)=1, \nonumber\\
    &\hspace{2.5cm} \mathbb{P}_v^\pi(\widehat{I}_\delta \neq I^\star(v)) \leq \delta \quad \forall v \in \mathcal{P}\}.
    \label{eq:delta-PAC-policies}
\end{align}
Here, and throughout the paper, we write $\mathbb{P}_v^\pi$ and $\mathbb{E}_v^{\pi}$ to denote probabilities and expectations under the instance $v$ and under the policy $\pi$. Prior works on fixed-confidence BAI show that 
$
    \inf_{\pi \in \Pi(\delta) } \mathbb{E}_v^\pi[\tau_\delta] \approx \Theta\left(\log \frac{1}{\delta}\right).
$ 
We anticipate that a similar growth rate for the expected stopping time holds in the context of our work. Our interest is to precisely characterise the asymptotic rate
\begin{equation}
    \liminf_{\delta \downarrow 0} \ \inf_{\pi \in \Pi(\delta) } \ \frac{\mathbb{E}_v^\pi[\tau_\delta]}{\log(1/\delta)},
    \label{eq:objective}
\end{equation}
where the asymptotics is as the error probability $\delta \downarrow 0$. 

\section{Lower bound}
\label{sec:lower-bound}
\vspace{-.1in}
In this section, we present an lower bound on \eqref{eq:objective} for any instance $v\in \mathcal{P}$. We first introduce the notion of sub-optimality gaps under this instance.
For any arm $i \in [K]$ and objective $m \in [M]$, we define the sub-optimality gap of the tuple $(i,m)$ under the instance $v$ as
\begin{align}
    \Delta_{i,m}(v)  \coloneqq \mu_{i_m^*(v),m} (v) - \mu_{i,m}(v),
    \label{eq:suboptimality-gap}
\end{align}
where, to recall, $i_m^*(v)$ is the best arm of objective $m$ under instance $v$. 

\begin{proposition}
\label{prop:lower_bound}
 Fix $\delta \in (0,1)$. For any $\delta$-PAC policy $\pi$,
\begin{equation}
    \mathbb{E}_{v}^{\pi}[\tau_\delta] \ge c^*(v) \, \log \left(\frac{1}{4\delta}\right), \quad{\forall v\in \mathcal{P}}
    \label{eq:lower-bound}
\end{equation}
where the constant $c^*(v)$ is given by
\vspace{-.1in}
\begin{align}
    c^*(v)^{-1} &\coloneqq \sup_{\omega \in \Gamma} \ \min_{m\in [M]} \  \min_{i\in[K] \setminus i_m^*(v)} \  \frac{\omega_i \, \omega_{i_m^*(v)} \, \Delta^2_{i,m}(v)}{2(\omega_i+\omega_{i_m^*(v)})}.
    \label{eq:cstar}
\end{align}
In \eqref{eq:cstar}, $\Gamma$ denotes the set of probability distributions on $[K]$, and we use the convention $\frac{\omega_i \, \omega_{i_m^*(v)}}{\omega_i+\omega_{i_m^*(v)}} = 0$ if $\omega_i = \omega_{i_m^*(v)} = 0$. Consequently, taking limits as $\delta \downarrow 0$ in \eqref{eq:lower-bound}, we get
\begin{equation}
    \liminf_{\delta \downarrow 0} \ \inf_{\pi \in \Pi(\delta)} \ \frac{\mathbb{E}_v^\pi[\tau_\delta]}{\log(1/\delta)} \geq c^*(v).
    \label{eq:lower-bound-asymptotic}
\end{equation}
\end{proposition}
\vspace{-.1in}
Proposition~\ref{prop:lower_bound} shows that the expected stopping time of any $\delta$-PAC policy $\pi$ grows as $\Omega(\log(1/\delta))$, as is the case in the prior works, and that the smallest (best) constant multiplying $\log(1/\delta)$ is $c^*(v)$ under the instance $v$. The constant $c^*(v)$ quantifies the complexity of identifying the best arms; notice that the smaller the sub-optimality gaps of the arms, the larger the value of $c^*(v)$, and therefore the larger the time required to find the best arms under any $\delta$-PAC policy.  When $M=1$, the expression in \eqref{eq:cstar} specialises to that of the problem complexity of BAI for Gaussian arms with unit variance; see, for instance, \citet{garivier2016optimal}. The proof of Proposition~\ref{prop:lower_bound} uses change-of-measure techniques introduced by \citet{garivier2016optimal} and is presented in Appendix~\ref{appndx:proof-of-lower-bound}.

Notice that $c^*(v)^{-1}$ is the value of a sup-min optimisation problem that is typically reminiscent of fixed-confidence BAI problems, as evident from the lower bounds in the prior works. 
Defining
\begin{equation}
    g_v(\omega) \coloneqq \min_{m\in [M]} \ \min_{i\in[K] \setminus i_m^*(v) } \ \frac{\Delta_{i,m}^2(v)}{2} \cdot \frac{\omega_i \, \omega_{i_m^*(v)}}{\omega_i+\omega_{i_m^*(v)}},
    \label{eq:g-v-omega}
\end{equation}
let $\omega^*(v) \in \arg\max_{\omega \in \Gamma} g_v(\omega) $ denote the optimal solution to \eqref{eq:cstar}. For the case $M=1$, the optimal solution $\omega^*(v)$ is referred to as the {\em oracle weight} of arm pulls (cf.~\citet{garivier2016optimal}), and represents the optimal proportion of times each arm must be pulled in the long run to achieve the lower bound in \eqref{eq:lower-bound}. Building on this insight, the well-known {\sc Track-and-Stop} ({\sc TaS}) algorithm of \citet{garivier2016optimal} and its variants such as Lazy {\sc TaS} \citep{jedra2020optimal} and Sticky {\sc TaS} \citep{degenne2019pure}, attempt to solve the sup-min optimisation of the lower bound therein to obtain the oracle weight for the estimated problem instance (in place of the unknown instance $v$) at each time step, a task that is computationally demanding. The computational burden of solving the sup-min optimisation at each time step is further escalated in the multi-objective setting of our work (i.e., when $M>1$). 

In this paper, we refrain from explicitly solving the sup-min optimisation for estimating the oracle weight at each time step. Instead, we construct a {\em surrogate proportion} as a proxy for the oracle weight, and sample arms according to the surrogate proportion at each time step. We show that the surrogate proportion may be computed easily by solving a linear program (using a standard technique such as the simplex method).

\begin{remark}
    In a recent publication, \citet{mukherjee2023best} introduced the concept of a {\em look-ahead distribution} as an approximation for the oracle weight. They devised an arm sampling strategy named Transport Cost Based Arms Elimination (or TCB in short), which samples arms at each time step according to the look-ahead distribution. While their algorithm demonstrates promising computational efficacy in the context of one-dimensional arm rewards ($M=1$), the applicability of their approach to scenarios with $M>1$ remains uncertain. In hindsight, we note that our scheme of sampling from surrogate proportions, although designed primarily for settings where $M>1$, may be easily specialised to settings where $M=1$. 
\end{remark}

{\color{black}
\begin{remark}
    Consider a simple MO-MAB setup with $M=3$, $K=2$, $\mu_{1,\cdot}=[100\ 0\ 50]^\top$, and $\mu_{2,\cdot}=[0\ 100\ 50+\varepsilon]^\top$, where $\varepsilon>0$ is small. For the above instance $v$, we have $I^*(v)=(1, 2, 2)$. Furthermore, the set of Pareto optimal arms is $\{1,2\}$; see Appendix~\ref{sec:differences_with_pfi} for a formula to compute the Pareto optimal arms. While the best arm of each objective is also Pareto optimal for the preceding instance, existing algorithms for Pareto optimal arm identification fail to assert that arm $2$ is the best arm for objective $3$, a task that is significantly complex courtesy of the small gap $\Delta_{2,3}=\varepsilon$. In contrast, Proposition~\ref{prop:lower_bound} demonstrates that the complexity of identifying arm $2$ as the best arm for objective $3$ scales as $\Omega(1/\varepsilon^2)$.
\end{remark}
}
\section{Achievability: Proposed Method}
\label{sec:achievability}
\vspace{-.1in}
In this section, we describe our computationally efficient algorithm named {\sc MO-BAI} based on the idea of {\em surrogate proportion} for pulling arms at each time step, and we also provide the pseudocode in Algorithm~\ref{alg:mo-bai} in the Appendix~\ref{appndx:Baseline}. Before we present the algorithm formally, we introduce some notations. Let $\widehat{\mu}_{i,m}(t)$ denote the empirical mean of rewards obtained from objective $m$ of arm $i$ up to time $t$, i.e., $
    \widehat{\mu}_{i,m}(t) \coloneqq \frac{1}{N_{i,t}} \sum_{s=1}^{t} \mathbf{1}_{\{A_s=i\}} \, r_{s,m}, $
and $N_{i,t}\coloneqq \sum_{s=1}^{t} \mathbf{1}_{\{A_s=i\}}$ denotes the number of times arm $i$ is pulled until time $t$. Let $\widehat{\Delta}_{i,m}(t) \coloneqq \widehat{\mu}_{\widehat{i}_m(t),m}(t)-\widehat{\mu}_{i,m}(t)$ denote the empirical gap of the tuple $(i,m) \in [K] \times [M]$, where $\widehat{i}_m(t) \in \arg\max_{\iota \in[K]} \widehat{\mu}_{\iota,m}(t) $ denotes the empirical best arm of objective $m$ at time $t$. Let $g_v(\cdot)$ be as defined in \eqref{eq:g-v-omega}, and for all $m \in [M]$ and $i\in [K]$, let
\begin{equation}
    g_v^{(i,m)}(\omega) \coloneqq  \frac{\Delta_{i,m}^2(v)}{2} \cdot  \frac{\omega_i \, \omega_{i_m^*(v)}}{\omega_i+\omega_{i_m^*(v)}}, \quad \omega \in \Gamma.
    \label{eq:def_g_i_m}
\end{equation}
Given any $\eta>0$, let 
$
    \Gamma^{(\eta)} \coloneqq \big\{\mathbf{\omega}\in \Gamma: ~ \forall i\in[K], ~ \mathbf{\omega}_i \ge \frac{\eta}{K  (1+\eta)} \big\},
$
and for all $\omega, \mathbf{z} \in \Gamma$, let
\begin{align}
    & h_v(\omega,\mathbf{z}) \coloneqq  \min_{m \in [M]} \ \min_{i \in [K] \setminus i_m^*(v)} \bigg\{ g^{(i,m)}_{v}(\mathbf{\omega}) + \langle  \nabla_{\omega} g^{(i,m)}_{v}(\mathbf{\omega}), \, \mathbf{z}- \mathbf{\omega}\rangle\bigg\}.
    \label{def:h}
\end{align}
With the above notations in place, we now describe the surrogate proportion and the associated arm selection rule of our algorithm.

\subsection{Surrogate Proportion and Arm Selection Rule}
\label{sec:surrogate_prop}
Let $l_t \coloneqq \max_{k \in \mathbb{N} : 2^k \le t} 2^k$, and let $\widehat{v}_t$ denote the empirical instance with means $ [\widehat{\mu}_{i,m}(t'): (i,m) \in [K] \times [M]]^\top$ for $t\in \mathbb{N}$ and $t'=\max\{t-1,K\}$.
The {\em surrogate proportion} at time step $t$, denoted $\mathbf{s}_t$, is defined as
\begin{equation}
    \mathbf{s}_t \coloneqq \arg\max_{\mathbf{s}\in \Gamma^{(\eta)}}  h_{\widehat{v}_{l_t}}( \widehat{\omega}_{\cdot,t-1},\mathbf{s}),
    \label{eq:def_s_t}
\end{equation}
where $\widehat{\omega}_{\cdot,t-1}$ is an empirical proportion that will be defined shortly. Notice that $\mathbf{s}_t$ is a probability distribution on the arms, with each component strictly positive ($\geq \frac{\eta}{K \, (1+\eta)} > 0$), thereby placing a strictly positive mass on each arm. 

To describe the empirical proportion $\widehat{\omega}_{\cdot, t}$, we introduce additional notations. Given $a \in [K]$, let $\mathbf{D}(a) \in \mathbb{R}^K$ denote the {\em one-hot} vector of length $K$ with a `$1$' in the $a$th component and `$0$' in the other components. Let $\mathbf{B}_{\cdot, t} = [\mathbf{B}_{i,t}: i \in [K]]^\top \in \mathbb{R}^K$ denote a generic vector of length $K$; in the sequel, we refer to $\mathbf{B}_{\cdot, t}$ as the {\em buffer} at time $t$. With the above notations in place, the quantity $\widehat{\omega}_{\cdot, t} = [\widehat{\omega}_{i, t}: i \in [K]]^\top \in \Gamma$ is defined via
\begin{equation}
    \widehat{\omega}_{i, t} \coloneqq \frac{N_{i,t} + \mathbf{B}_{i,t}}{t}.
    \label{eq:hat-omega-t-1-definition}
\end{equation}

{\bf Arms selection rule \ } We now describe the arms selection rule of our algorithm. To begin with, each arm is pulled once in the first $K$ time steps $t=1,\ldots,K$. Initialising $B_{i,t}=0$  for all $i \in [K]$, for all $t \le K$, the algorithm computes $\widehat{\omega}_{\cdot, t}$ according to \eqref{eq:hat-omega-t-1-definition}. Subsequently, the algorithm computes the surrogate proportion $\mathbf{s}_t$ using \eqref{eq:def_s_t}, and pulls arm $A_t$ according to the rule
\begin{equation}
    A_t \in \argmax_{i \in [K]} \  [\mathbf{B}_{\cdot,t-1} + \mathbf{s}_t]_i,
    \label{eq:arms-selection-rule}
\end{equation}
where $[\mathbf{x}]_i$ represents the $i$-th component of $\mathbf{x}$.
Following this, the algorithm iteratively updates the buffer according to the rule
\begin{equation}
    \mathbf{B}_{\cdot, t} = \mathbf{B}_{\cdot, t-1} - \mathbf{D}(A_t) + \mathbf{s}_t.
    \label{eq:buffer-update-rule}
\end{equation}
 The above process repeats until the stoppage. Some remarks are in order.
\begin{remark}
\label{rem:exact-computatation-of-convex-program-vs-easy-computation-of-LP}
    When $M=1$, the well-known {\sc Track-and-Stop} algorithm of \citet{garivier2016optimal} traditionally computes the oracle weight $\widetilde{\omega}_{\cdot, t} = \argmax_{\omega \in \Gamma} g_{\widehat{v}_{t}}(\omega)$ and samples arm $A_t$ guided by $\widehat{\omega}_{\cdot, t}$ at each time step $t$. Solving the preceding maximisation problem at each time step becomes computationally demanding, as exemplified by the authors therein, which is further exacerbated in scenarios where $M>1$. {\color{black} In fact, it is easy to show that for any instance $v \in \mathcal{P}$, the maximisation problem $\sup_{\omega \in \Gamma} g_v(\omega)$ is a convex program (noting that the function $g_v$ defined in \eqref{eq:g-v-omega} is concave; see Appendix~\ref{lemma:g_v_concave} for the details) with $O(K)$ variables and $O(MK)$ constraints. However, solving this convex program {\em exactly} through an iterative algorithm (such as the ellipsoid method) would necessitate running infinitely many iterations, thus rendering this approach practically infeasible. Of course, in practical implementations, one may only need to solve the convex program {\em approximately}; however, in this case, providing theoretical guarantees would also be commensurately more challenging. } 
    
     {\color{black} In contrast, our proposed strategy for computing surrogate proportions in \eqref{eq:def_s_t} offers computational efficiency by maximising $h_v$ defined in \eqref{def:h} which serves as a surrogate version of $g_v$ (hence the name surrogate proportion for $\mathbf{s}_t$). Notice that $h_v$ is a linear function of its second argument. Thus, the optimisation problem in \eqref{eq:def_s_t} is a simple {\em linear program} that may be solved {\em efficiently} and {\em exactly} using standard techniques such as the simplex method that typically takes polynomial time~\citep{spielman2004smoothed}. In this manner, we sidestep the intricacies of solving for the oracle weight at each time step, hence making our approach well-suited for scenarios with $M>1$.}
\end{remark}

\begin{remark}
The astute reader will observe a striking similarity between the expression for $h_v$ and that of the linear approximation of $g_v$, specified around any point $\omega$ by $g_v(\omega)+\langle \nabla_\omega g_v(\omega), \mathbf{z}-\omega \rangle$ for $\mathbf{z}$ lying in a neighborhood of $\omega$. However, the two expressions are vastly different. 
While optimising the linear approximation of $g_v$ in place of $h_v$ to compute surrogate proportions, it might possibly lead to the convergence arising from the surrogate proportions to a sub-optimal weight in the long run. We show that our approach leads to the convergence of the empirical proportion to an almost-optimal weight for which the evaluation of $g_v$ matches the constant $c^*(v)$ of the lower bound up to a factor of $1+\eta$. See Section~\ref{sec:proof-sketch} for more details.

\end{remark}

\begin{remark}
The sequence $\{l_t\}_{t=1}^{\infty}$ is strategically designed to minimize frequent alterations to the surrogate function $h_{\widehat{v}_{l_t}}$, thereby facilitating precise control over the {\em estimation error}---the difference between $h_{\widehat{v}_{l_t}}$ and $g_{\widehat{v}_{l_t}}$---in our analytical framework. Detailed insights can be found in Lemma~\ref{lemma:g_to_c_star}. Additionally,  $\eta>0$ is judiciously chosen to prevent the estimation error from diverging to infinity.
\end{remark}

\vspace{-.1in}
\subsection{Stopping and Recommendation Rules}
\label{subsec:stopping-rule}

We now delineate the stopping and recommendation rules employed in our algorithm. We adopt a variant of Chernoff's stopping rule~\citep {Kaufmann2016, lattimore_szepesvari_2020}; our design is particularly inspired by~\citet{chen2022federated}.   Specifically, let 
\begin{equation}
    Z(t) \coloneqq  \min_{m\in [M]} \ \min_{i\in[K] \setminus \widehat{i}_m(t)} \  \frac{N_{i,t} \,  N_{\widehat{i}_m(t),t} \, \widehat{\Delta}^2_{i,m}(t)}{2(N_{i,t}+ N_{\widehat{i}_m(t),t})}
    \label{eq:test-statistic}
\end{equation}
denote a test statistic at time $t$. We define the stopping time of our algorithm via
\begin{equation}
    \tau_\delta = \min\{t \geq K: Z(t) \!>\! \beta(t, \delta)\},
 \label{eq:chernoff-stopping-time}
\end{equation}
where the threshold $\beta(t, \delta) \coloneqq MK \log(t^2+t) + f^{-1}(\delta)$. Here, $f:(0,+\infty) \rightarrow (0,1)$ is defined as
\begin{equation}
   f(x) \coloneqq \sum_{i=1}^{MK} \frac{x^{i-1}e^{-x}}{(i-1)!}, \quad x \in (0, +\infty).
   \label{eq:deff}
\end{equation}
As for the recommendation rule, for each objective, our algorithm outputs the arm with the highest empirical mean corresponding to the objective as the best arm of the objective. That is, $\widehat{I}_\delta=[\widehat{i}_1(\tau_\delta), \ldots, \widehat{i}_M(\tau_\delta)]^\top \in [K]^M$, where
$
\widehat{i}_m(\tau_\delta) \in \arg\max_{i\in[K]} \widehat{\mu}_{i,m}(\tau_\delta)
$
for each $m\in [M]$. 

Integrating the arms selection rule (with a fixed parameter $\eta>0$), stopping rule, and recommendation rule delineated above, our algorithm for multi-objective best arm identification, abbreviated as \textsc{MO-BAI}, is presented in Algorithm~\ref{alg:mo-bai}.

\begin{remark}
      Note that $\textsc{MO-BAI}$ selects each arm once in the first $K$ time slots. That is, $N_{\cdot,K} = [1, \ldots, 1]^\top$.  Together with \eqref{eq:hat-omega-t-1-definition} and \eqref{eq:buffer-update-rule}, it implies that $\widehat{\omega}_{\cdot, t} \in \Gamma^{(\eta)}$ for all $t > K$.
\end{remark}

\subsection{Performance of {\sc MO-BAI}}
\label{sec:upperbound}
In this section, we characterise the performance of {\sc MO-BAI} with a fixed input parameter $\eta$.  The first result below asserts that {\sc MO-BAI} is $\delta$-PAC for any $\delta\in (0,1)$. 
\begin{proposition}
\label{prop:delta_pac}
    Fix $\eta>0$ and $\delta \in (0,1)$. Then, {\sc MO-BAI} with parameter $\eta$ is $\delta$-PAC, i.e., $\forall v \in \mathcal{P}$
\begin{align}
    \mathbb{P}_v^{\textsc{MO-BAI}} \left(\tau_\delta < +\infty \right)& =1  \quad \text{and} \\
    \mathbb{P}_v^{\textsc{MO-BAI}} \big(\hat{I}_\delta = I^*(v) \big) &\ge 1-\delta. \label{eq:err-prob-at-most-delta}
\end{align}
\end{proposition}
We present the proof in Appendix~\ref{sec:proof_of_delta_pac}. It is worth noting that the form of $Z(t)$ and the threshold $\beta(t, \delta)$ play important roles in the proof. The following result demonstrates an asymptotic upper bound on the stopping time of  {\sc MO-BAI}.

\begin{theorem}
\label{thm:upperbound}
Under {\sc MO-BAI} with parameter $\eta > 0$, $\forall v \in \mathcal{P}$,
\begin{align}
    &\limsup_{\delta \downarrow 0} \frac{\mathbb{E}_v\left[\tau_\delta \right]}{\log(\frac{1}{\delta})} \le (1+\eta) \, c^*(v) \quad\mbox{and}
    \label{eq:asymptotic-upper-bound-expected-stopping-time} \\
    &\mathbb{P}_v^{\textsc{MO-BAI}} \left(\limsup_{\delta \downarrow 0} \frac{\tau_\delta}{\log(\frac{1}{\delta})} \le (1+\eta) \, c^*(v) \right) =1. \
    \label{eq:asymptotic-upper-bound-almost-surely-finite-stopping-time}
\end{align}
\end{theorem}
We present the proof in Appendix~\ref{sec:proof_of_upper_bound}. Observe that in the asymptotic limit of $\delta \downarrow 0$, the upper bound in~\eqref{eq:asymptotic-upper-bound-expected-stopping-time} matches the lower bound in~\eqref{eq:lower-bound-asymptotic} up to a factor $1+\eta$. The parameter $\eta$ can be set arbitrarily close to $0$, thereby establishing the asymptotic optimality of {\sc MO-BAI}.

{\color{black}
\begin{remark}
    Our novel idea of surrogate proportion and the MO-BAI algorithm are inspired, at a high level, by the gradient-based algorithms in \cite{jaggi2013revisiting,wang2021fast,menard2019gradient}. Yet, notably, while our algorithm operates with $M \geq 1$ objectives, the algorithm in \cite{wang2021fast} is specifically designed to operate with a {\em single} objective ($M=1$). Adapting the algorithm in \cite{wang2021fast} to our setting, while maintaining computational tractability and considering the exponential $O(K^M)$ scaling of possible sets of best arms across $M$ objectives, presents challenges in preserving asymptotic optimality (Theorem~\ref{thm:upperbound}). Furthermore, even for the case when $M=1$, our algorithm differs substantially from that in \cite{wang2021fast}. The latter functions on the concept of an {\em $r$-subdifferential subspace} (an idea that is in turn adapted from~\cite{ravi2019deterministic}) and involves dealing with an infinite sequence of hyperparameters $\{r_i\}$ that must be chosen carefully. In contrast, our algorithm operates {\em 
 without the need to choose/tune these hyperparameters}.
\end{remark}
}

\section{Proof Sketch of Theorem~\ref{thm:upperbound}}
\label{sec:proof-sketch}

In this section, we outline the key ideas that go into the proof of Theorem~\ref{thm:upperbound}. At a high level, the proof encompasses two important steps: (1) Deriving a bound on the limiting value of $Z(t)/t$, and (2) Deriving a upper bound on the stopping time $\tau_\delta$ based on the limiting value of $Z(t)/t$. Throughout, we fix an underlying instance $v$.

\subsection{Bounding the Limiting Value of \texorpdfstring{$Z(t)/t$}{Z(t)/t}}

The key to deriving a ``good'' upper bound for the limiting value of $Z(t)/t$ lies in establishing that the evaluations of $g_v$ at the oracle weights arising from \textsc{MO-BAI} (with input parameter $\eta$) match, in the long run, with the constant $c^*(v)$ appearing in the lower bound up to a factor of $1+\eta$. Towards this, we introduce the following auxiliary constant $\forall \eta>0$:
\begin{equation}
    C(v,\eta) \coloneqq  \sup_{\substack{\omega, \mathbf{y}\in \Gamma^{(\eta)}, \\ \gamma\in (0,1), \\ \mathbf{z} = \omega + \gamma(\mathbf{y} - \omega)}} \frac{2}{\gamma^2 } \bigg(h_v\big(\omega, \ \mathbf{z}-\omega\big) - g_v\big(\mathbf{z}\big)\bigg).
    \label{eq:def_C}
\end{equation}
The above definition appears to bear close resemblance with the definition for the curvature of a function, but they are not equivalent (see details in Appendix~\ref{sec:proof_of_upper_bound}). We show in Lemma~\ref{lemma:finite_cuvature} that $C(v,\eta)< +\infty$ for all $\eta>0$. For all $\eta>0$, let 
\begin{equation}
    \tilde{c}(v,\eta)^{-1} \coloneq \sup_{\omega \in \Gamma^{(\eta)}} g_v(\omega).
    \label{eq:tilde-c-v-eta}
\end{equation}
In Lemma~~\ref{lemma:relax_c}, we show that $ c^*(v) \le (1+\eta) \, \tilde{c}(v,\eta)$. With the above notations in place, the the key result of this section is presented below.

\begin{replemma}{lemma:g_to_c_star}

Fix $\eta>0$. For all $t_1,t_2 \in \mathbb{N}$ with $t_2>t_1 > K$, under $\textsc{MO-BAI}$ with input parameter $\eta$, we have
\begin{align}
    \lvert \tilde{c}(v,\eta)^{-1}-g_v(\widehat{\omega}_{\cdot,t_2}) \rvert \le  \frac{t_1}{t_2} \, \tilde{c}(v,\eta)^{-1}  + 11 \, \epsilon_{t_1}(v) 
    +  \frac{2 \, \log(t_2) \, \overline{C}_{t_1}(\eta)}{t_2}
    \label{eq:relation-between-tilde-c-and-g-v2}
\end{align} 
almost surely, where for any time step $t$:
\begin{itemize}
    
    \item $\epsilon_t(v) \coloneqq \sup_{t' \ge l_t} \sup_{\omega \in \Gamma^{(\eta)}} \big\lvert g_v(\omega)-g_{\widehat{v}_{t'}}(\omega) \big\rvert$.

    \item  $\overline{C}_t(\eta) \coloneqq \sup_{t' \ge l_t} C(\widehat{v}_{t'},\eta)$.
\end{itemize} 
\end{replemma}
We gather from Lemma~\ref{lemma:g_to_c_star} that $\lim_{t_1 \to \infty} \epsilon_{t_1}(v)=0$ and  $\limsup_{t_1 \to \infty} \overline{C}_{t_1}(\eta) < +\infty$ almost surely; we present the details in Appendix~\ref{sec:proof_of_upper_bound}. Setting $t_2 = t_1^{-\lambda}$ for some $\lambda \in (0,1)$, we demonstrate that as $t_1 \rightarrow \infty$, the right-hand side of~\eqref{eq:relation-between-tilde-c-and-g-v2} converges to $0$ almost surely, which in turn implies that $g_v(\widehat{\omega}_{\cdot,t_2})$ converges to $\tilde{c}(v,\eta)^{-1}$ in the limit as $t_2 \to \infty$ almost surely.
This leads to the following quantitative characterisation of the limiting value of $Z(t)/t$.
\begin{replemma}{lemma:zt_t_to_c_tilde}
    Fix $\eta>0$ and consider the non-stopping version of $\textsc{MO-BAI}$ with input parameter $\eta$ (i.e., a policy in which the stopping rule corresponding to Lines 11-14 in Algorithm~\ref{alg:mo-bai} are not executed). Under this policy,
    \begin{equation*}
        \lim_{t \rightarrow \infty} \frac{Z(t)}{t} = \tilde{c}(v,\eta)^{-1}
    \quad \text{almost surely}.
    \end{equation*}
\end{replemma}

\subsection{Bounding the Stopping Time \texorpdfstring{$\tau_\delta$}{taudelta}}

Using the limiting value of $Z(t)/t$ derived earlier, we upper bound the stopping time almost surely and in expectation. First, we introduce two auxiliary terms.

(1) ${T_{\rm gap}}(v, \eta, \epsilon)$:
  For any $\epsilon>0$, let $T_{\rm gap}(v,\eta,\epsilon)$ denote the smallest positive integer-valued random variable such that
\begin{equation*}
     \left\lvert \frac{Z(t)}{t} -\tilde{c}(v,\eta)^{-1} \right\rvert \le \epsilon \quad \forall \ t \ge T_{\rm gap}(v,\eta,\epsilon).
\end{equation*}
Thanks to Lemma~\ref{lemma:zt_t_to_c_tilde}, we have that $T_{\rm gap}(v,\eta,\epsilon)$ is finite almost surely. We show in Lemma~\ref{lemma:T_gap_finite} that $T_{\rm gap}(v,\eta,\epsilon)$ also has finite expectation.

(2) $ T_{\rm thres}(v, \eta, \epsilon, \delta)$: For $\delta \in (0,1)$ and $\eta>0$, let
\begin{align}
   & T_{\rm thres}(v,\eta, \epsilon, \delta) \coloneqq 1+ \frac{f^{-1}(\delta)}{\tilde{c}(v,\eta)^{-1} - \epsilon} + \frac{MK}{\tilde{c}(v,\eta)^{-1} - \epsilon} \, \log\left( \left(\frac{2f^{-1}(\delta)}{\tilde{c}(v,\eta)^{-1} - \epsilon}\right)^2 + \frac{2f^{-1}(\delta)}{\tilde{c}(v,\eta)^{-1} - \epsilon}  \right) .
   \nonumber
\end{align}
The right-hand side of the above expression is carefully designed to meet the following requirement.

\begin{replemma}{lemma:Tlast}
Fix $\delta \in (0,1)$ and $\eta>0$, and consider the threshold $\beta(t, \delta)$ in {\sc MO-BAI}. For all $\epsilon \in \big(0, \tilde{c}(v,\eta)^{-1} \big)$, there exists $\delta_{\rm thres}(v,\eta,\epsilon) > 0$ such that for all $\delta \in (0, \delta_{\rm thres}(v,\eta, \epsilon))$,
\begin{equation*}
    t \tilde{c}(v,\eta)^{-1} > \beta(t,\delta) + \epsilon t \quad \forall t \ge T_{\rm thres}(v,\eta, \epsilon, \delta).
\end{equation*}
\end{replemma}
Using Lemma~\ref{lemma:Tlast}, we show that for any $\epsilon \in \big(0, \tilde{c}(v,\eta)^{-1} \big)$ and $\delta \in (0, \delta_{\rm thres}(v,\eta, \epsilon))$, almost surely,
\begin{align}
 \tau_{\delta} 
 & \le T_{\rm gap}(v,\eta,\epsilon)+T_{\rm thres}(v,\eta,\delta, \epsilon)+ K +1.
\nonumber
\end{align}
This crucial step leads  to $
\limsup_{\delta \downarrow 0} \tau_{\delta}\cdot \frac{\tilde{c}(v,\eta)^{-1} - \epsilon}{f^{-1}(\delta)} \le 1 
$ almost surely, because $T_{\rm gap}(v,\eta,\epsilon)$ is independent of $\delta$ and that 
$
    \lim_{\delta \downarrow 0} T_{\rm thres}(v,\eta,\delta, \epsilon) \cdot  \frac{\tilde{c}(v,\eta)^{-1} - \epsilon}{ f^{-1}(\delta)}  = 1.
$ 
Finally, using the fact that $ \lim_{\delta \downarrow 0} \frac{f^{-1}(\delta)}{\log(1/\delta)} = 1 $, and letting $\epsilon \downarrow 0$, we arrive at \eqref{eq:asymptotic-upper-bound-expected-stopping-time} and \eqref{eq:asymptotic-upper-bound-almost-surely-finite-stopping-time}. It is noteworthy that our proof techniques are applicable in a wide range of pure exploration problems with single-dimensional and multi-dimensional rewards from arms.
\section{Numerical Study}

\label{sec:numerical}
\begin{figure}[!ht]
    \centering
    \includegraphics[width = .6\columnwidth]{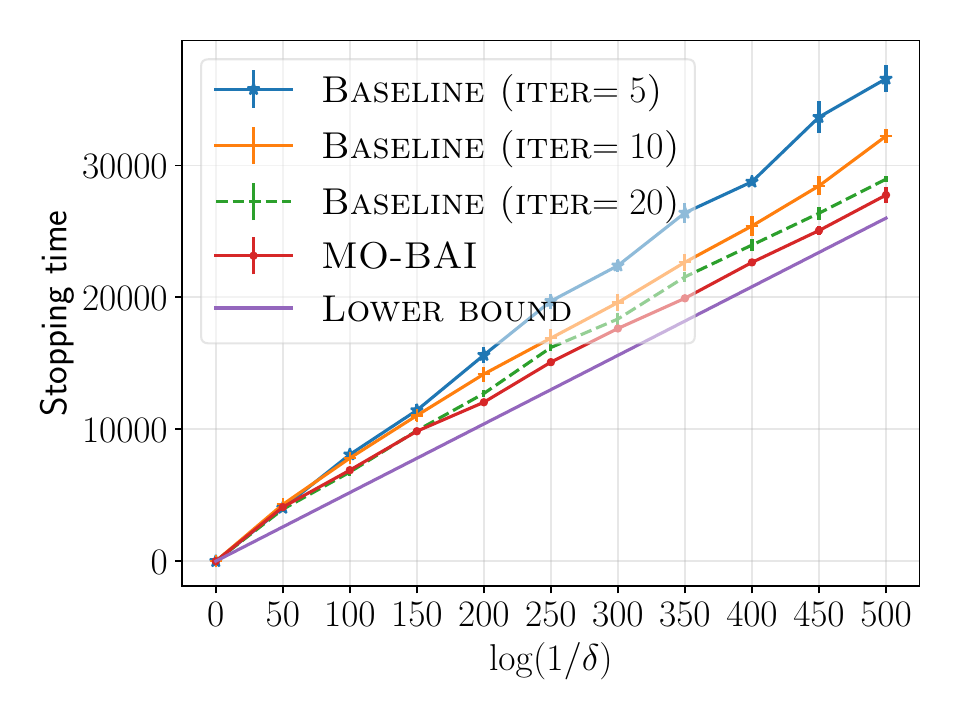}
    \caption {Plot of average stopping times of \textsc{MO-BAI} and \textsc{Baseline} (with varying iteration numbers) for the synthetic dataset.}
    \label{fig:synthetic-stopping-times}
\end{figure}
We run experiments to validate the effectiveness of \textsc{MO-BAI} through empirical assessments on the SNW dataset~\citep{zuluaga2016} and a synthetic dataset, and their detailed descriptions are presented in Appendix~\ref{sec:dataset_description}. Specifically, we compare our algorithm against \textsc{Baseline}, a multi-objective adaptation of the D-Tracking algorithm by \citet{garivier2016optimal} which was originally designed for a single objective (see details in Appendix~\ref{app:baseline_details}), and a Successive Elimination-based~\citep{even2006action} algorithm. As alluded to in Remark~\ref{rem:exact-computatation-of-convex-program-vs-easy-computation-of-LP}, the implementation of \textsc{Baseline} involves solving an optimization problem, which, in scenarios with $M>1$, may not be exactly resolved.
 Hence, we employ an iterative method (details in Appendix~\ref{appndx:Baseline} and pseudo code of Algorithm~\ref{alg:optimisation-sub-routine-in-baseline}), and the accuracy of the obtained solution depends on the number of iterations.  

\begin{table*}[!ht]
\centering
\begin{tabular}{|c|ccc|}
\hline
\multirow{2}{*}{{\sc MO-BAI}} & \multicolumn{3}{c|}{{\sc Baseline} with {\sc iter} Iteration Steps}                                                     \\ \cline{2-4} 
                        & \multicolumn{1}{c|}{\textsc{iter} = 5}              & \multicolumn{1}{c|}{\textsc{iter} = 10}             & \textsc{iter} = 20              \\ \hline
$\mathbf{10.0 \pm 7.2}$         & \multicolumn{1}{c|}{$41.0\pm 10.0$} & \multicolumn{1}{c|}{$92.2\pm 19.8$} & $184.3\pm 41.0$ \\ \hline
\end{tabular}

\caption {Average computation times (in ms) for a single execution of the optimisation routines in {\sc MO-BAI} and {\sc Baseline} on Apple M1 Chip with 16GB of RAM. }
    \label{tab:synthetic-computation-times}
\end{table*}
In addition, for the fairness of comparison, we employ the threshold $\widehat{c}_t(\delta)=\log((1+\log t)/\delta)$ in both {\sc MO-BAI} and {\sc Baseline} in the numerical study, and note that this is different from our theoretical threshold $\beta(t, \delta)$. Additional details regarding the modified threshold $\widehat{c}_t(\delta)$ and its justification in the context of our work can be found in Appendix~\ref{sec:Threshold}.

{\bf Results } A comparison of \textsc{MO-BAI} with $\eta=0.1$ and \textsc{Baseline} with varying iteration counts ($\textsc{iter}=5,10,20$) for the synthetic dataset is depicted in Figure~\ref{fig:synthetic-stopping-times}. The error bars in the figures are obtained from running three independent trials. 
Further, the computation times (in ms) for a single execution of the optimization routines in \textsc{Baseline} and \textsc{MO-BAI} are shown in Table~\ref{tab:synthetic-computation-times}. Clearly, the figures demonstrate the superior performance of {\sc MO-BAI} compared to \textsc{Baseline}. As the number of iteration steps in \textsc{Baseline} increases, its performance approaches that of \textsc{MO-BAI}, an artefact of improved accuracy in solving the optimization routine in \textsc{Baseline}. However, this improvement comes at a significant escalated computational cost, as evidenced by Table~\ref{tab:synthetic-computation-times}. These results, as well as the other experimental results of the SNW dataset in Appendix~\ref{appndx:Baseline} underscore the efficacy of   {\sc MO-BAI}.


\section{Conclusions and Future Work}

This work considered a novel best arm identification setting in which a single pull of an arm yields an $M$-dimensional vector as its reward. The goal was to identify the $M$ best arms, one corresponding to each dimension, under the fixed-confidence regime. We developed an efficient algorithm based on the original idea of {\em surrogate proportions}, that we proved is asymptotically optimal and computationally efficient. We conducted empirical studies on a synthetic dataset and the SNW datasets to substantiate the proposed algorithm's computational efficiency and asymptotic optimality. 
Our results are asymptotically optimal in the sense that the results are tight as the error probability $\delta \downarrow 0$. It would be fruitful to investigate whether the ideas in non-asymptotic strengthenings of single-objective pure exploration problems (e.g., \citet{degenne2019nonasymptotic}) carry over to our multi-objective setting. 


\newpage

\bibliography{main}
\bibliographystyle{apalike}


\newpage
\appendix
\onecolumn

\section{More Details on Numerical Study}
\label{appndx:Baseline}

In this section, we provide more details on our experimental implementation.
Especially, we implement the algorithm of {\sc MO-BAI} as shown in the pseudocode of Algorithm~\ref{alg:mo-bai}. 

\subsection{Simulation Environment}   \label{app:env}
The experiments were executed on an Apple M1 Chip with 16 GB of memory, operating on Mac OS 14.2.1. The linear programming procedures outlined in Algorithm~\ref{alg:mo-bai} and Algorithm~\ref{alg:optimisation-sub-routine-in-baseline} were executed using SciPy v1.11.3 within the Python 3.11.2 environment.

\subsection{Descriptions of the Datasets}
\label{sec:dataset_description}
\textbf{Synthetic Dataset: \ } Our synthetic dataset is generated with parameters $K=20$ and $M=10$. For all pairs $(i, m)$ where $i \neq m$, $\mu_{i,m}$ is uniformly chosen from the interval $[0,1]$. For pairs $(i, m)$ where $i = m$, $\mu_{i,m}$ is uniformly selected from $[1.2,2]$. These values remain constant throughout the experiment. Let $v = [\mu_{i,m}: (i, m) \in [K] \times [M]]^\top$. It is evident that $i_m^*(v)=m$ for every $m \in [M]$. Additionally, $\Delta_{i,m}(v)> 0.2$ for all $i \neq i^*_m(v)$.

\textbf{SNW Dataset: \ } We adopt the SNW dataset introduced by \citet{zuluaga2016}, consisting of 206 distinct hardware implementations of a sorting network. Following the protocol outlined by \citet{ararat2023vector}, the objective values, represented by the negative of the area, serve as mean rewards for the designs, and Gaussian noises are added to the mean rewards in the bandit dynamic. Consequently, in this dataset, we have $K=206$ and $M=2$. To facilitate the simulation, we scale the rewards of each arm by a factor of $10$.

\subsection{Results for the SNW Dataset}
\label{appndx:experimental-results}
We fix $\eta=0.1$ in our implementation of {\sc MO-BAI.}
We run three independent trials, and average the stopping times from these trials to obtain the values in Table~\ref{tab:snw}. The tabulated results indicate the superior performance of our proposed \textsc{MO-BAI} algorithm over \textsc{Baseline} on the SNW dataset. Notably, this dataset involves a greater number of decision variables compared to the synthetic dataset instance (i.e., 206 versus 20), posing increased difficulty in running the optimization routine of {\sc Baseline}. Consequently, to achieve comparable performance with \textsc{MO-BAI}, it is necessary to increase the number of iteration steps in \textsc{Baseline}. This underscores the superiority of \textsc{MO-BAI} from a practical standpoint. 

\begin{table*}[!ht]
    \centering {
    \begin{tabular}{|l|c|c|c|}
        \hline
        $\log(1/\delta)$ & $10$  & $50$  & $100$ \\ 
        \hline
        {\sc MO-BAI} &  {\bf 1,686.67 $\pm$ 21.03} & {\bf 5,097.33 $\pm$ 90.78} & {\bf 8,435.0 $\pm$ 115.23}  \\
        \hline
        {\sc Baseline (iter=100)} & $7,428.0\pm 139.69$ & $18,612.67 \pm 214.59$ &  $30,914.33 \pm 622.61$ \\ 
        \hline
        {\sc Baseline (iter=20)} & $10,584.67 \pm 313.52$  &  $37,316.33 \pm 406.90$&  $65,395.0 \pm 602.98$ \\ 
        \hline
    \end{tabular} }
    
    \caption{Comparison of the empirical stopping times for the SNW dataset \citep{zuluaga2016}.}
    \label{tab:snw}
\end{table*}

\begin{table*}[!ht]
    \centering{ 
    \begin{tabular}{|c|c|c|}
        \hline
       
         & $\delta = 0.1$ & $\delta = 0.05$  \\
        \hline
        {\sc MO-BAI} & \bf{968.82 $\pm$ 58.21}  &  {\bf 1023.77 $\pm$ 67.42} \\
        {\sc Baseline}   & $4485.98 \pm 124.92$	      & $6168.29 \pm 132.01$   \\
         {\sc Baseline-Non-Unif  } & $3841.05 \pm 136.44$	      & $4320.55\pm  128.26$  \\
          {\sc MO-SE} & $2322.39\pm 461.54$	      & $2411.16 \pm 421.88$ \\
        \hline

\end{tabular}
}

\caption{Average stopping times obtained by running $100$ independent trials with practical error probability $\delta=0.1$ and $\delta=0.05$ for the  SNW dataset \citep{zuluaga2016}. In {\sc Baseline} and {\sc Baseline-Non-Unif}, we set $\textsc{iter}=20$.}
\label{tab:practical-100-trials}
\end{table*}

\subsection{Curated Threshold for Simulations}
\label{sec:Threshold}
It is customary in the fixed-confidence BAI literature to employ  thresholds in simulations that differ from theoretical thresholds. Notably, in the single-objective case, \citet{garivier2016optimal} utilized $\beta_{\rm GK}^{\rm empirical}(t, \delta) = \log((1+\log t)/\delta)$ for empirical evaluation, a threshold that is (a log factor) smaller than the theoretical threshold $\beta_{\rm GK}^{\rm theoretical}(t, \delta) = \log(C_K t^\alpha/\delta)$ employed in their D-Tracking algorithm. Here, $\alpha>1$ is a parameter of the D-Tracking algorithm, and $C_K$ is a universal constant that depends only on the number of arms $K$. In a recent work, \citet{kaufmann2021mixture} introduced the concept of ``rank'' for pure exploration problems. They demonstrated that for a problem with rank $R$, the threshold $\widehat{c}_t(\delta) = 3R \, \log(1+\log(t/R)) + O(\log(1/\delta))$, when combined with the GLR test statistic, yields a $\delta$-PAC stopping rule; cf. \citet[Proposition~21]{kaufmann2021mixture}. The authors therein note that BAI with a single objective is a problem of rank $2$. It is noteworthy that the same holds true of multi-objective BAI problems; for a formal proof of this, see Appendix~\ref{appndx:mo-bai-problem-has-rank-2}. In light of the above rationale, we adopt the threshold $\log((1+\log t)/\delta)$ for our simulations following~\citet{garivier2016optimal}, and note that is different from our theoretical threshold $\beta(t, \delta)$ defined in Section~\ref{subsec:stopping-rule}.

{\color{black} Even with the modified threshold of $\log((1+\log t)/\delta)$, we observe near-zero empirical error rates for {\sc MO-BAI} and {\sc Baseline} algorithms in our experiments, thereby implying that this modified threshold is still conservative in practice. This conforms with the heuristics presented in \cite[Section~6]{garivier2016optimal}, providing additional practical justification for adopting the modified threshold in our experiments.}

\begin{algorithm}[!h]
    \caption{Multi-Objective Best Arm Identification (MO-BAI)}
    \begin{algorithmic}[1]
        \REQUIRE ~~\\
        $\delta\in (0,1)$: confidence level.\\
        {\color{black}$\eta>0$: relaxation parameter.}
         
        \ENSURE $\widehat{I}_\delta$: the best arms. 
         \STATE Pull each arm once. 
         \STATE Initialise the buffer $\mathbf{B}_{i,t}=0$ for all $i \in [K]$ and $t\in[K]$.
        
         \FOR{ $t\in \{K+1,K+2,\ldots\} $}
             \STATE Compute the empirical mean $\widehat{\mu}_{i,m}(t)$ for each $(i,m) \in [K] \times [M]$.
                     
             \STATE Compute the current empirical proportion 
            $$
                \widehat{\omega}_{i,t-1} \coloneqq \frac{N_{i,t-1}+\mathbf{B}_{i,t-1}}{t-1}.
            $$
            
            \STATE Set $l_t \leftarrow \max_{k \in \mathbb{N} : 2^k \le t} 2^k$.
            
            \STATE Compute the surrogate proportion {\color{black} for instance $\widehat{v}_{l_t}$} via 
            $$
                \mathbf{s}_t= \arg\max_{\mathbf{s}\in \Gamma^{(\eta)}}  h_{\widehat{v}_{l_t}}( \widehat{\omega}_{\cdot,t-1},\mathbf{s}).
            $$

            \STATE Set $t \leftarrow  t + 1$.
            
            \STATE Pull arm $A_t \in \arg\max_{i\in[K]}  [\mathbf{B}_{\cdot,t-1} + \mathbf{s}_t]_i$.
            
            \STATE Update the Buffer $\mathbf{B}_{\cdot,t} \leftarrow \mathbf{B}_{\cdot,t-1} + \mathbf{s}_t-\mathbf{D}(A_t) $.
                
            \IF{$Z(t)>\beta(t, \delta)$}
                \STATE $\widehat{I}_\delta \leftarrow \text{the empirical best arms at time }t$.
                \STATE break.
            \ENDIF
        \ENDFOR
        \RETURN Best arms $\widehat{I}_\delta$.
    \end{algorithmic}
    \label{alg:mo-bai}
\end{algorithm}

\subsection{The {\sc Baseline} Algorithm and its Implementation Details}
\label{app:baseline_details}
The pseudo-code of {\sc Baseline} is presented in Algorithm~\ref{alg:Baseline}. It is important to note that the approach proposed by~\citet{garivier2016optimal} for solving the optimization problem in line $6$ of {\sc Baseline}  becomes impractical when $M>1$ due to the various best arms across different objectives. In our implementation of {\sc Baseline}, we adopt the sub-routine in Algorithm~\ref{alg:optimisation-sub-routine-in-baseline} to solve the optimisation problem in Line~6 of Algorithm~\ref{alg:Baseline} by specifying the number of iterations steps $\textsc{iter}$, and its convergence can be established by the idea of Lemma~\ref{lemma:g_to_c_star}. Figure~~\ref{fig:synthetic-stopping-times} and Tab.~\ref{tab:synthetic-computation-times} show respectively the stopping times and computation times (in ms) incurred under $\textsc{iter} \in \{5,10,20\}$.
\begin{algorithm}[!h]
    \caption{{\sc Baseline} (Multi-objective adaptation of D-Tracking~\citep{garivier2016optimal})}
    \begin{algorithmic}[1]
        \REQUIRE ~~\\
        $K \in \mathbb{N}$: number of arms \\ 
        $\delta\in (0,1)$: confidence level.\\
        {\sc it}: number of iteration steps
        \ENSURE $\widehat{I}_\delta$: the best arms. 
         \STATE Pull each arm once. 
        
         \FOR{ $t\in \{K+1,K+2,\ldots\} $}
         \IF{$\min_{i\in [K]} N_{i,t-1} < \sqrt{(t-1) / K}$ } 
        \STATE Pull arm $A_t \in \arg\min_{i\in [K]}N_{i,t-1}$; resolve ties uniformly.     
        \ELSE    
            \STATE Compute the empirical oracle weight 
            $$
                \tilde{\omega}_{\cdot,t} = \textsc{SubRoutine}(K, \textsc{it}, \widehat{v}_t).
            $$
            \STATE Pull arm $A_t \in \arg\max_{i\in[K]} N_{i,t-1} - t \, \tilde{\omega}_{i,t}$. 
        \ENDIF
            \IF{$Z(t)>\log\big((1+\log t\big)/\delta)$}
                \STATE $\widehat{I}_\delta \leftarrow$ the empirical best arms of time $t$.
                \STATE break.
            \ENDIF
        \ENDFOR
        \RETURN Best arms 
 $\widehat{I}_\delta$.
    \end{algorithmic}
    \label{alg:Baseline}
\end{algorithm}

\begin{algorithm}[!h]
    \caption{Sub-routine to solve the optimisation in Line~6 of Algorithm~\ref{alg:Baseline} -- $\textsc{SubRoutine}(K, \textsc{it}, \widehat{v})$}
    \begin{algorithmic}[1]
        \REQUIRE ~~\\
        $K \in \mathbb{N}$: number of arms. \\
        {\sc it}: number of iteration steps. \\
        $\widehat{v}$: empirical instance.
        
        \ENSURE $\widehat{\omega}$: the oracle weight. 

        \STATE Initialise $\widehat{\omega} = [1/K, \ldots, 1/K]^\top$,  $\mathbf{\widetilde{N}}=[\widetilde{N}_i: i \in [K]]^\top = \mathbf{0}$.
        
         \FOR{ $k \in \{1, \ldots, \textsc{it}\} $}
            \STATE Compute $\mathbf{s}_k = \argmax_{\mathbf{s} \in \Gamma} \ h_{\widehat{v}}(\widehat{\omega}, \mathbf{s})$.
            \STATE Set $\mathbf{\widetilde{N}} \leftarrow \mathbf{\widetilde{N}}+ \mathbf{s}_k$
            \STATE Update $\widehat{\omega} \leftarrow \frac{\mathbf{\widetilde{N}} }{k}$.
        \ENDFOR
        \RETURN Oracle weight $\widehat{\omega}$.
    \end{algorithmic}
    \label{alg:optimisation-sub-routine-in-baseline}
\end{algorithm}
Furthermore, the threshold used in the implementation of {\sc Baseline} is equal to the single-objective empirical threshold of $\log((1+\log t)/\delta)$ used in \citet{garivier2016optimal}. Notably, this threshold remains independent of the number of objectives $M$. The rationale behind this choice stems from the fact that the ``rank'' of a multi-objective BAI problem with $M$ independent objectives is equal to the rank of the single-objective BAI problem for all values of $M$. See Appendix~\ref{appndx:mo-bai-problem-has-rank-2} for further details.

{\color{black}
As such, the {\sc Baseline} algorithm, with any value of $\textsc{iter} < +\infty$, is not asymptotically optimal (though practically implementable), while asymptotically optimal but practically not implementable for $\textsc{iter} = +\infty$. A plausible scheme to achieve asymptotic optimality, while ensuring practical feasibility, is to let $\textsc{iter}$ grow with $t$. For e.g., if $\textsc{iter}(t) = O(\log t)$, solving for $\widetilde{\omega}_{\cdot, t} = \argmax_{\omega \in \Gamma} g_{\widehat{v}_t}(\omega)$ up to a $1/\mathrm{poly}(t)$ error at time step $t$ requires $O(\log t)$ iterations. However, quantifying the exact growth rate (e.g., $\textsc{iter}(t) = O(\log(t)), O(\sqrt{t}), O(t)$, or $O(\exp(t))$) that is necessary to achieve asymptotic optimality is a technically challenging task;  the latter involves quantifying the approximation error of each subroutine as {\sc iter} grows with $t$, and ensuring that these errors amortize asymptotically as $t \to \infty$. Moreover, if $ITER(t)$ growth as $t$ (e.g., $ITER(t)=\sqrt{t}$), the number of iterations at each time step will go to infinite as $t \rightarrow \infty$, while the number of iterations at each time step is still finite as $t \rightarrow \infty$ in {\sc MO-BAI}.

Nonetheless, our experiments on both synthetic and the SNW datasets indicate that the stopping times typically remain below $10^5$, implying that $\textsc{iter} = \log_2 t \leq 20$ at all times $t$ for these datasets. As shown in Table~\ref{tab:practical-100-trials}, the {\sc MO-BAI} algorithm outperforms the {\sc Baseline} algorithm with {\sc iter} set to 20, on the SNW dataset. This suggests that, in practice, there is limited benefit to allowing {\sc iter} to grow with $t$.
}

{\color{black}
In addition, to enhance the comprehensiveness of our comparative analyses, we investigated alternative implementations of Algorithm~\ref{alg:optimisation-sub-routine-in-baseline} by modifying the initialization method of $\widehat{\omega}$ in Line~1 to $\widehat{\omega}=\tilde{\omega}_{\cdot,t-1}$ (i.e., the estimate of $\omega$ from the previous time step) at time step $t$ instead of setting it to be the uniform distribution $[1/K, \ldots, 1/K]^\top$. We call this variant {\sc Baseline-non-unif}, and present the experimental outcomes in Table~\ref{tab:practical-100-trials}. The empirical findings underscore that  our  proposed {\sc MO-BAI} is  significantly better than  {\sc Baseline-non-uniform} on the SNW dataset.
}

{\color{black}
\begin{algorithm}[!h]
    \caption{{\sc MO-SE} (Multi-objective adaptation of Successive Elimination~\citep{even2006action})}
    \begin{algorithmic}[1]
        \REQUIRE ~~\\
        $K \in \mathbb{N}$: number of arms \\ 
        $\delta\in (0,1)$: confidence level.\\
        {\sc it}: number of iteration steps
        \ENSURE $\widehat{I}_\delta$: the best arms. 
          \STATE Initialize $t=0$
         \FOR{ $ m\in \{1,2,\ldots, M\} $}
         \STATE Set $\mathcal{S}=\{1,2,\ldots, K\}$
         \WHILE{ $\lvert \mathcal{S} \rvert >1$ }
         \STATE Pull each arm in $\mathcal{S}$ once.
         \STATE $t \leftarrow t + \lvert \mathcal{S} \rvert$
         \STATE $\alpha_t=\sqrt{2\ln \left(4 MK t^2 / \delta\right) / t}$
         \STATE Eliminate all the arms $i$ in $\mathcal{S}$ with $\max_{j\in \mathcal{S}} \widehat{\mu}_{j,m}(t) - \widehat{\mu}_{i,m}(t) > 2\alpha_t$
         \ENDWHILE
        \STATE $\widehat{i}_m \leftarrow $ the only arm in $\mathcal{S}$
        \ENDFOR
        \RETURN Best arms 
 $\widehat{I}_\delta = (\widehat{i}_1,\ldots, \widehat{i}_M)$.
    \end{algorithmic}
    \label{alg:MO-SE}
\end{algorithm}

\subsection{Multi-Objective Successive Elimination ({\sc MO-SE}) and its Implementation Details}
We also implement a multi-objective version of a classical algorithm for BAI--Successive Elimination~\citep{even2006action}. This algorithm, which we call Multi-Objective Successive Elimination ({\sc MO-SE})  is shown formally in Algorithm~\ref{alg:MO-SE}. Specifically, in {\sc MO-SE}, there are $M$ rounds, and we determine the empirical best arm of $m-$th objective in $m-$th round using the principle of successive elimination. In particular, we set $\alpha_t=\sqrt{2\ln \left(4 MK t^2 / \delta\right) / t}$ in Algorithm~3 of~\citet{even2006action}, which is a natural adaption to the multiobjective case as there are total $MK$ arms and the noises are Gaussian with unit variance in our setting. 
From a theoretical standpoint, {\sc MO-SE} is not (asymptotically) optimal even in the case of $M=1$, which is clearly inferior to our {\sc MO-BAI}. This can also be clearly seen via
the experimental results  shown in Table~\ref{tab:practical-100-trials}, which again empirically underscores the superiority of our proposed {\sc MO-BAI} over all considered baselines.
}

\subsection{Impact of \texorpdfstring{$\eta$}{eta} on Performance of {\sc MO-BAI}}
We run {\sc MO-BAI} on the synthetic dataset introduced in Section~\ref{sec:numerical} for $\eta \in \{2.0,1.0,0.5,0.1\}$. The results are shown in Figure~~\ref{fig:synthetic_eta}.
We observe that the performance for $\eta=0.1$ is superior to that for $\eta > 0.1$, This observation aligns with our theoretical findings. Indeed, because $N_{i,t}/t \approx \widehat{\omega}_{i,t}$ for all large $t$ (noting that $|B_{i,t}| \leq 1$), and $\min_{i \in [K]} \widehat{\omega}_{i,t} \geq \frac{\eta}{K(1+\eta)}$, it is evident that the fraction of times each arm is pulled in the long run increases with increase in $\eta$ (as $\eta \mapsto \eta/(1+\eta)$ is an increasing function), thereby leading to larger stopping times.  Furthermore, when the stopping time is not excessively large, the performances of $\eta=0.1$ and $\eta=0.5$ empirically demonstrate a notable degree of similarity. This phenomenon may be attributed to the empirical mean necessitating a greater number of arm pulls for stabilization. 

\begin{figure}[!ht]
    \centering
    \includegraphics[width = .7\columnwidth]{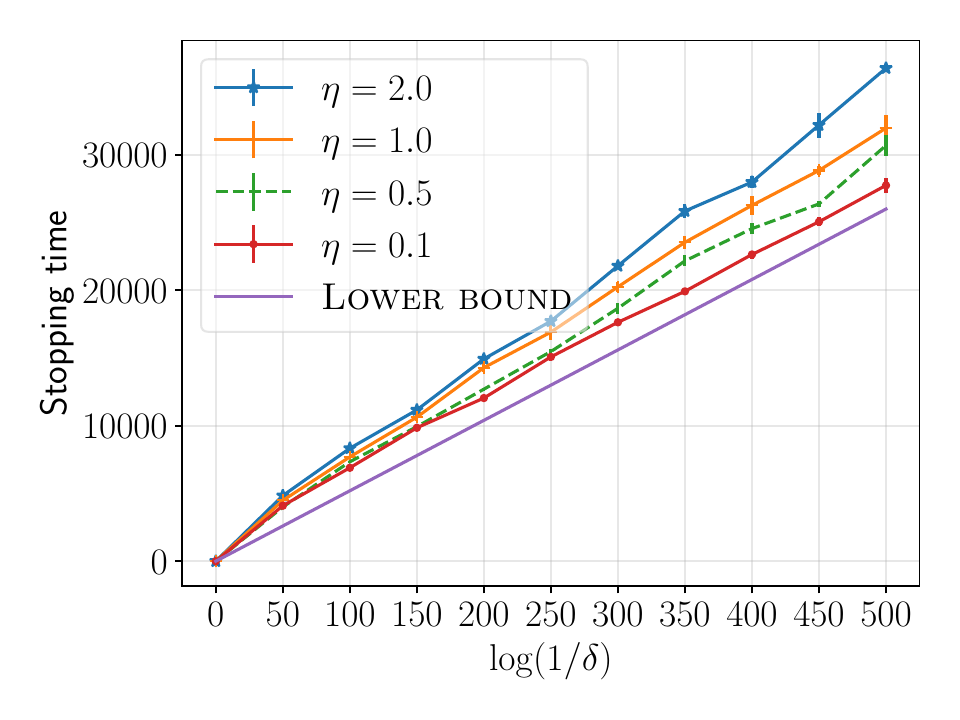}
    \caption {Comparison of the empirical stopping times with varying values of $\eta$ for the synthetic dataset.}
    \label{fig:synthetic_eta}
\end{figure}

\section{Multi-Objective BAI vs Pareto Frontier Identification}
\label{sec:differences_with_pfi}
Our research framework shares structural similarities with those used in Pareto frontier identification, yet the tasks of our investigation are distinctly different. For clarity and ease of illustration in this section, we consider that the arm means are distinct across each objective, specifically, $\mu_{i,m} \neq \mu_{j,m}$ for all $m \in [M]$ and $i \neq j$. Then, according to the extant literature~\citep{auer2016pareto,ararat2023vector,kim2023pareto} on Pareto frontier identification, an arm $i$ is defined as Pareto optimal if, for every other arm $j$ where $j \neq i$, there exists at least one objective $m \in [M]$ for which $\mu_{i,m} > \mu_{j,m}$. Consequently, under this definition, the best arm for each objective (as defined in our paper) is inherently Pareto optimal.

However, typical Pareto Frontier Identification algorithms are designed merely to determine the Pareto optimality of each arm without the capability to specifically identify the optimal arm for any given objective $m \in [ M]$. Conversely, while our proposed algorithm effectively ascertains the Pareto optimality of the best arm {\em for each objective}, it does not ensure the identification of all Pareto optimal arms.

Thus, the nature of our research task diverges fundamentally from that of traditional Pareto Frontier Identification, reflecting distinct analytical goals and methodological requirements.

\section{Rank of Multi-Objective Best Arm Identification}
\label{appndx:mo-bai-problem-has-rank-2}

In this section, we show that the rank of multi-objective BAI problem with $M$ independent objectives is equal to $2$ for all values of $M$; notably, this is also the rank of the single-objective BAI problem. Before we present the formal arguments, we reproduce the definition of rank of a pure exploration problem from \citet{kaufmann2021mixture}.

\begin{definition}\citep[Definition~20]{kaufmann2021mixture}
    \label{def:rank-of-pure-exploration-problem}
    Fix constants $d, P, Q, R \in \mathbb{N}$. A sequential identification problem specified by a partition $\mathcal{O}=\bigcup_{p=1}^{P} \mathcal{O}_p$, where $O_p \subseteq \mathbb{R}^d$ for all $p$, is said to have rank $R$ if for every $p \in \{1, \ldots, P\}$,
    \begin{equation}
        \mathcal{O} \setminus \mathcal{O}_p = \bigcup_{q=1}^{Q} \bigg\{\boldsymbol{\lambda} \in \mathbb{R}^d: (\lambda_{k_1^{p,q}}, \ldots, \lambda_{k_R^{p,q}}) \in \mathcal{L}_{p, q}\bigg\}
        \label{eq:rank-definition}
    \end{equation}
    for a family of indices $k_r^{p,q} \in [d]$ and open sets $\mathcal{L}_{p,q}$ indexed by $r \in [R]$, $q \in [Q]$, and $p \in [P]$. In other words, the problem has rank $R$ if for each $p$, the set $\mathcal{O} \setminus \mathcal{O}_p$ is a finite union of sets that are each defined in terms of only $R$-tuples.
\end{definition}

\subsection{Rank of Single-Objective BAI Problem}

Consider the classical BAI problem with a single objective $(M=1)$ and $K$ arms, specified by a problem instance $\boldsymbol{\mu} = [\mu_1, \ldots, \mu_K]^\top \in \mathbb{R}^K$. In this case, we have $d=P=K$, and for each $p \in [K]$, the set $\mathcal{O}_p = \{\boldsymbol{\lambda} \in \mathbb{R}^d: a^\star(\boldsymbol{\lambda})=p\}$ consists of all problem instances with best arm $p$. Furthermore, for each $p \in [K]$,
$$
    \mathcal{O} \setminus \mathcal{O}_p = \bigcup_{q \neq p} \bigg\{\boldsymbol{\lambda} \in \mathbb{R}^K: \lambda_q > \lambda_p\bigg\},
$$
thereby implying that rank $R=2$.

\subsection{Rank of Multi-Objective BAI Problem}
Consider now a multi-objective BAI problem with $M$ independent objectives and $K$ arms. In this case, a problem instance is specified by $\boldsymbol{\mu} = [\mu_{i,m}: (i, m) \in [K] \times [M]] \subseteq \mathbb{R}^{KM}$. We thus have $d = KM$. Also, for any $p = [p_1, \ldots, p_M]^\top \in [K]^M$, we have 
$$
    \mathcal{O}_p = \{\boldsymbol{{\lambda}} \in \mathbb{R}^{KM}:~ \forall m \in [M], ~~ p_m = \text{best arm of objective }m \text{ under the instance }\boldsymbol{\lambda}\},
$$
thereby implying that $P = K^{M}$. Furthermore,
\begin{equation}
    \mathcal{O} \setminus \mathcal{O}_p = \bigcup_{m=1}^{M} \ \bigcup_{q = [q_1, \ldots, q_M]^\top \in [K]^M} \bigg\{\boldsymbol{\lambda} = [\lambda_{i,m}: (i, m) \in [K] \times [M]]^\top \in \mathbb{R}^{KM}: \lambda_{q_m, m} > \lambda_{p_m, m}\bigg\},
    \label{eq:rank-of-mobai}
\end{equation}
thereby implying that $R=2$.
\section{Proof of Proposition~\ref{prop:lower_bound}}
\label{appndx:proof-of-lower-bound}

Firstly, we introduce a useful lemma adapted from~\citet{Kaufmann2016}.
\begin{lemma}
\label{lemma:useful-lemma}
Fix $\delta>0$ and a $\delta$-PAC policy $\pi$ with stopping time $\tau_\delta$. Let $\mathcal{F}_{\tau_\delta} = \sigma(\{(X_{A_t,m}(t), \, A_t) : t\in [\tau_\delta], m\in[M]\})$ denote the history of all the arm pulls and rewards seen up to the stopping time $\tau_\delta$ under the policy $\pi$. Then, for any pair of instances $v, v' \in \mathcal{P}$ with arm means $\{\mu_{i,m}: i \in [K], m \in [M]\}$ and $\{\mu_{i,m}^\prime: i \in [K], m \in [M]\}$ respectively, and any $\mathcal{F}_{\tau_\delta}$-measurable event $E$,
\begin{align}
	\sum_{i=1}^K \ \sum_{m=1}^M \ \mathbb{E}_{v}^\pi\left[N_{i,\tau_\delta}\right] \, \frac{(\mu_{i,m} - \mu_{i,m}^\prime)^2}{2} \ge  d_{\mathrm{KL}}\left(\mathbb{P}_{v}^\pi(E), \mathbb{P}_{{v'}}^{\pi}(E)\right),
	\label{eq:transportation-lemma}
\end{align} 
where $D_{\mathrm{KL}}(p \| q)$ denotes the Kullback--Leibler (KL) divergence between distributions $p$ and $q$, and $d_{\mathrm{KL}}(x,y)$ denotes the KL divergence between two Bernoulli distributions with parameters $x$ and $y$. 
\end{lemma}
The proof of Lemma \ref{lemma:useful-lemma} follows along the same lines as in the proof of~\citet[Lemma 19]{Kaufmann2016}, and is hence omitted. We then note the following lower bound derived from 
Lemma~\ref{lemma:useful-lemma}.

\begin{lemma}
\label{lemma:lower_bound_0}
    Fix $\delta>0$ and instance $v \in \mathcal{P}$. For any $\delta$-PAC policy $\pi$ with stopping time $\tau_\delta$,
    $$
        \mathbb{E}_{v}^{\pi}[\tau_\delta] \ge \dfrac{\log(\frac{1}{4\delta})}{\sup_{\omega \in \Gamma} \  \inf_{v'\in {\rm Alt}(v)}\! \sum_{i=1}^{K} \,  \omega_{i} \, \sum_{m=1}^M \frac{\big(\mu_{i,m}(v)\!-\!\mu_{i,m}(v')\big)^2}{2}},
    $$
    where ${\rm Alt}(v)$ denotes the set of problem instances with a set of best arms distinct from the set of best arm under $v$.
\end{lemma}
\begin{proof}
Fix a $\delta$-PAC policy $\pi$. Let
$$
E^* \coloneqq \{ \widehat{I}_\delta = I^*(v) \},
$$
where $\widehat{I}_\delta$ denotes the set of best arms (one for each objective) output by policy $\pi$ at stoppage. By Lemma~\ref{lemma:useful-lemma}, we have 
\begin{align}
  \sum_{m=1}^M \  \sum_{i=1}^K \mathbb{E}_{v}^\pi\left[N_{i,\tau_\delta} \right] \,  \frac{(\mu_{i,m}(v)-\mu_{i,m}(v'))^2}{2} \ge  d_{\mathrm{KL}}\left(\mathbb{P}_{v}^\pi(E^*), \mathbb{P}_{{v'}}^{\pi}(E^*)\right).\label{eq:lowbound_c1}
  \end{align}
Because $\pi$ is $\delta$-PAC, we have  $\mathbb{P}_{v}^\pi(E^*) \ge 1-\delta$ and $\mathbb{P}_{v'}^\pi(E^*) \le \delta$. This in turn implies that
\begin{equation}
\label{eq:lowbound_c2}
    d_{\mathrm{KL}}\left(\mathbb{P}_{v}^\pi(E^*), \mathbb{P}_{{v'}}^{\pi}(E^*)\right) \ge \log \Big(\frac{1}{4\delta}\Big).
\end{equation}
Combining~\eqref{eq:lowbound_c1} and~\eqref{eq:lowbound_c2}, we obtain
  \begin{align}
   \sum_{m=1}^M \  \sum_{i=1}^K \mathbb{E}_{v}^\pi\left[N_{i,\tau_\delta} \right] \, \frac{(\mu_{i,m}(v)-\mu_{i,m}(v'))^2}{2} \ge  \log\Big(\frac{1}{4\delta}\Big) \label{proof_lowerboud:step1},
\end{align}

Letting $\bar{\omega}_{i} \coloneqq \frac{\mathbb{E}_{v}^\pi\left[N_{i,\tau_\delta} \right]}{\mathbb{E}_{v}^\pi[ \tau_\delta]} $, we have
\begin{equation}
    \mathbb{E}_{v}^\pi [ \tau_\delta ]   \ge  
     \frac{\log(\frac{1}{4\delta})}{ \sum_{m=1}^M\  \sum_{i=1}^K \bar{\omega}_{i} \frac{(\mu_{i,m}(v)-\mu_{i,m}(v'))^2}{2}}. 
     \label{proof_lowerboud:step2}
\end{equation}
Noting that~\eqref{proof_lowerboud:step2} holds for any $v' \in {\rm Alt}(v)$, we have
\begin{equation}
\mathbb{E}_{v}^\pi [\tau_\delta]  \ge  
  \frac{\log(\frac{1}{4\delta})}{\inf_{v' \in {\rm Alt}(v)} \sum_{m=1}^M\  \sum_{i=1}^K \bar{\omega}_{i} \frac{(\mu_{i,m}(v)-\mu_{i,m}(v'))^2}{2}} . \label{proof_lowerboud:step3}
\end{equation}
Finally, using the fact that $\tilde{\omega} = [\tilde{\omega}_i: i \in [K]]^\top \in \Gamma$, we get
\begin{equation}
\mathbb{E}_{v}^\pi[  \tau_\delta]   \ge  
  \frac{\log(\frac{1}{4\delta})}{\sup_{\omega \in \Gamma} \inf_{v' \in {\rm Alt}(v)} \sum_{m=1}^M\  \sum_{i=1}^K \omega_{i} \frac{(\mu_{i,m}(v)-\mu_{i,m}(v'))^2}{2}}. \label{proof_lowerboud:step4}
\end{equation} 
 This completes the proof.
\end{proof}

\begin{lemma}
\label{lemma:c_star_simplicatoin}
    Fix $v\in \mathcal{P}$ with arm means $\{\mu_{i,m}: i \in [K], m \in [M]\}$. For any $\mathbf{\omega} \in \Gamma$,
    \begin{equation}
        g_v(\omega) =   \inf_{v'\in {\rm Alt}(v)} \ \sum_{i=1}^{K} \,  \omega_{i} \sum_{m=1}^M \frac{\big(\mu_{i,m} -\!\mu_{i,m}^\prime\big)^2}{2},
        \label{eq:simplication}
    \end{equation}
    where $\mu_{i,m}^\prime$ denotes the mean of arm $i$ corresponding to objective $m$ under the instance $v'$, and $g_v(\cdot)$ is defined in~\eqref{eq:g-v-omega}.
\end{lemma}
\begin{proof}
Note that
\begin{align} 
    & \inf_{v'\in {\rm Alt}(v)} \  \sum_{m=1}^{M} \  \sum_{i=1}^K  \omega_i\frac{\left(\mu_{i,m}-\mu_{i,m}^\prime\right)^2}{2} \nonumber\\
    &= \min_{m \in [M]}  \min_{i\in[K] \setminus i^*_m(v)} \  \inf_{v':\mu_{i,m}^\prime > \mu_{i^*_m(v),m}^\prime} \  \sum_{\jmath=1}^{M}\  \sum_{\imath=1}^K  \omega_\imath\frac{(\mu_{\imath,\jmath}-\mu_{\imath,\jmath}^\prime)^2}{2} \nonumber\\
    &= \min_{m \in [M]} \ \min_{i\in[K] \setminus i^*_m(v)} \  \inf_{v':\mu_{i,m}^\prime \ge \mu_{i^*_m(v),m}^\prime}\   \left[\omega_i \, \frac{(\mu_{i,m}-\mu_{i,m}^\prime)^2}{2} +\omega_{i^*_m(v)}\frac{(\mu_{i^*_m(v),m}-\mu_{i^*_m(v),m}^\prime)^2}{2}\right] \label{eq:g_v(omega)-simplification_last}\\
    &= \min_{m\in [M]} \ \min_{i\in[K] \setminus i^*_m(v) } \ \frac{\Delta_{i,m}^2(v)}{2} \, \left(\frac{\omega_i \, \omega_{i^*_m(v)}}{\omega_i+\omega_{i^*_m(v)}} \right)
    \label{eq:g_v(omega)-simplification} \\
    &= g_v(\omega), \nonumber
\end{align}
where \eqref{eq:g_v(omega)-simplification} follows by using the method of Lagrange multipliers and noting that the inner infimum in \eqref{eq:g_v(omega)-simplification_last} is attained at 
\begin{align}
    \mu_{i,m}^\prime
    &= \mu_{i,m} + \left(\mu_{i^*_m(v),m}-\mu_{i,m} \right)\frac{\omega_{i^*_m(v)}}{\omega_i+\omega_{i^*_m(v)}}  ,\nonumber\\
    \mu_{i^*_m(v),m}^\prime 
    &= \mu_{i^*_m(v),m} - \left(\mu_{i^*_m(v),m} - \mu_{i,m} \right) \, \frac{\omega_{i}}{\omega_i+\omega_{i^*_m(v)}} .
    \label{eq:optimisers-of-inner-infimum}
\end{align}
This completes the proof.
\end{proof}

\begin{proof}[Proof of Proposition~\ref{prop:lower_bound}]

Finally, by combining Lemmas~\ref{lemma:lower_bound_0} and~\ref{lemma:c_star_simplicatoin}, we see that Theorem~\ref{prop:lower_bound} holds.
\end{proof}

\section{Proof of Proposition \ref{prop:delta_pac}}
\label{sec:proof_of_delta_pac}
Below, we first record some useful results that will be used in the proof.

\begin{lemma}{\citep[Lemma 33.8]{lattimore_szepesvari_2020}}
    Let $Y_1,Y_2,\ldots$ be independent Gaussian random variables with mean $\mu$ and unit variance. Let $\widehat{\mu}_n \coloneqq \frac{1}{n} \sum_{i=1}^n Y_i$. Then,
    \[
    \mathbb{P} \left( \exists\,  n \in \mathbb{N}: \frac{n}{2}(\widehat{\mu}_n-\mu)^2 \ge \log(1/\delta) + \log(n(n+1)) \right) \le \delta.
    \]
    \label{lemma:latimaore_338}
\end{lemma}

\begin{lemma}\citep[Lemma A.4]{chen2022federated}
    Fix $n \in \mathbb{N}$. Let $Y_1,Y_2,\ldots,Y_n$ be independent random variables with $\mathbb{P}(Y_i \le y) \le y$ for all $y\in[0,1]$ and $i\in[n]$. Then, for any $\epsilon > 0$,
    \begin{equation}
    \label{eq:lemmaf}
    \mathbb{P} \bigg( \sum_{i=1}^n \log(1/Y_i) \ge \epsilon \bigg) \le f^{(n)}(\epsilon)
    \end{equation}
    where $f^{(n)}:(0,+\infty) \rightarrow (0,1) $ is defined by 
    \[
        f^{(n)}(x) = \sum_{i=1}^{n} \frac{x^{i-1}e^{-x}}{(i-1)!}, \quad x \in (0, +\infty).
    \]
    \label{lemma:mydeltaunion}
\end{lemma}

\begin{proof}[Proof of Proposition \ref{prop:delta_pac}]

Fix any confidence level $\delta \in (0,1)$, problem instance $v \in \mathcal{P}$, and $\eta>0$.  Consider $\textsc{MO-BAI}$.  We claim that $\tau_\delta < + \infty$ almost surely; a proof of this is deferred until the proof of Lemma~\ref{lemma:finite_stopping}. For $m\in[M]$ and $i\in [K]$, let 
\begin{equation}
    \xi_{i,m} \coloneqq  \sup_{t \ge K} \frac{N_{i,t}}{2}\big(\widehat{\mu}_{i,m}(t)-\mu_{i,m}(v)\big)^2 - \log\big(N_{i,t}(N_{i,t}+1)\big).
    \label{eq:xi-i-m}
\end{equation}
From Lemma \ref{lemma:latimaore_338}, we know that for any confidence level $\delta' \in (0,1)$,
\begin{align}
    \mathbb{P}_v^{\textsc{MO-BAI}} (\xi_{i,m} \ge \log(1/\delta')) \le \delta'.
\end{align}
Let $\xi'_{i,m} \coloneqq \exp(-\xi_{i,m})$. From Lemma \ref{lemma:mydeltaunion}, we know that for any $\epsilon>0$,
\begin{align}
    & \mathbb{P}_v^{\textsc{MO-BAI}} \left(\sum_{m\in [M]}\  \sum_{i \in [K]}\  \log(1/\xi'_{i,m}) \ge \epsilon \right) \le f^{(MK)}(\epsilon) \nonumber \\
    & \stackrel{(a)}{\implies} \mathbb{P}_v^{\textsc{MO-BAI}} \left(\sum_{m\in [M]}\  \sum_{i \in [K]}\  \xi_{i,m} \ge \epsilon \right) \le f^{(MK)}(\epsilon) \nonumber \\
    & \stackrel{(b)}{\implies} \mathbb{P}_v^{\textsc{MO-BAI}} \left(\sum_{m\in [M]}\  \sum_{i \in [K]}\  \xi_{i,m} \ge \epsilon \right) \le f(\epsilon) \nonumber \\
    & \stackrel{(c)}{\implies} \mathbb{P}_v^{\textsc{MO-BAI}} \left(\sum_{m\in [M]}\  \sum_{i \in [K]}\  \xi_{i,m} \ge f^{-1}(\delta)  \right) \le \delta,  
    \label{eq:combinationbound} 
\end{align}
definition of $\xi'_{i,m}$,  $(b)$ follows from the definition of $f$ in \eqref{eq:deff}, and in writing $(c)$, we (i) make use of the fact that $f$ is continuous and strictly decreasing and therefore admits an inverse, and (ii) set $\epsilon = f^{-1}(\delta)$.
Plugging the expression for $\xi_{i,m}$ from \eqref{eq:xi-i-m} in \eqref{eq:combinationbound}, and noting that $N_{i,t} \leq t$ for all $i \in [K]$ and $t \geq K$, we get
\begin{align}
    & \mathbb{P}_v^{\textsc{MO-BAI}} \bigg(\forall\, t\ge K \  \sum_{m\in [M]} \sum_{i \in [K]}
    \frac{N_{i,t}}{2}\big(\widehat{\mu}_{i,m}(t)-\mu_{i,m}(v)\big)^2  \le  MK\log\big(t(t+1)\big) + f^{-1}(\delta) \bigg) \ge 1-\delta \nonumber \\ 
    & \implies \mathbb{P}_v^{\textsc{MO-BAI}} \bigg(\forall\, t\ge K \ \sum_{m\in [M]} \sum_{i \in [K]}
    \frac{N_{i,t}}{2}\big(\widehat{\mu}_{i,m}(t)-\mu_{i,m}(v)\big)^2  \le  \beta(t, \delta) \bigg) \ge 1-\delta.
    \label{eq:importanceforpac}
\end{align}
Note that at the stopping time $\tau_\delta = \tau_\delta$, by definition, we must have
\[
    Z(\tau_\delta) = \inf_{v' \in {\rm Alt}(\widehat{v}(\tau_\delta))} \sum_{m \in [M]} \sum_{i \in [K]} N_{i,\tau_\delta} \frac{(\mu_{i,m}(v') - \widehat{\mu}_{i,m}(\tau_\delta))^2}{2} > \beta(\tau_\delta, \delta).
\]
Thus, \eqref{eq:importanceforpac} may be expressed equivalently as 
$$
    \mathbb{P}_v^{\textsc{MO-BAI}} \bigg( v \notin {\rm Alt} \big(\widehat{v}(\tau_\delta)\big) \bigg) \ge 1 - \delta \quad \Longleftrightarrow \quad \mathbb{P}_v^{\textsc{MO-BAI}}\bigg( I^*(v) = I^*\big(\widehat{v}(\tau_\delta)\big) \bigg) \ge 1 - \delta,
$$ 
thereby establishing the desired result. This completes the proof.
\end{proof}


\section{Proof of Theorem \ref{thm:upperbound}}
\label{sec:proof_of_upper_bound}
Let $v \in \mathcal{P}$ be fixed. Firstly, we define the curvature~\citep{jaggi2013revisiting} of a concave function.

\begin{definition}[Curvature]
    Given a concave function $f:\mathcal{D} \rightarrow \mathbb{R}$ defined on a convex domain $\mathcal{D}$, the {\em curvature} of $f$ is defined as
    \begin{equation}
        C_{\rm cur}(f) \coloneqq \sup_{\substack{\omega, \mathbf{y} \in \mathcal{D}, \\ \gamma \in (0,1), \\ \mathbf{d} \in \partial f(\omega)}} \frac{2}{\gamma^2} \bigg( f(\omega)+ \langle \mathbf{z} - \omega, \mathbf{d} \rangle - f \big(\mathbf{z}\big) \bigg), \quad \mathbf{z} = \omega + \gamma(\mathbf{y} - \omega),
        \label{eq:def_Cur}
    \end{equation}
    where $\partial f(\omega)$ denotes the super-differential of $f$ at $\omega$.
\end{definition}
In the following, we will show that $g^{(i,m)}_v(\cdot)\big|_{\Gamma^{(\eta)}}$ (the function $g^{(i,m)}_v(\cdot)$ restricted to the set $\Gamma^{(\eta)}$) is a concave function. It is worth noting that the curvature of $g_v$ is a function of its super-gradients, whereas the constant $C(v,\eta)$ is a function of the gradient of $g^{(i,m)}_v(\cdot)$ for $i \neq i^*_m(v)$ (see \eqref{def:h}). While $C_{\rm cur}(g_v)$ may be infinitely large, we will show that $C(v,\eta)$ is finite for all $\eta>0$. The latter, as we shall see, is an important property that will enable us to demonstrate the asymptotic optimality of our policy.


Before we proceed further, we record some useful results.
\begin{lemma}{(Adapted from ~\citet[Lemma 7]{jaggi2013revisiting})}
\label{Lemma:jaggi}
Let $f: \mathcal{D} \to \mathbb{R}$ be a concave and differentiable function defined over a convex domain $\mathcal{D}$, with a gradient function $\nabla f$ is Lipschitz-continuous w.r.t. some norm $\|.\|$ over the domain $\mathcal{D}$ with a Lipschitz constant $L>0$. Then,
$$
C_{\rm cur}(f) \leq \operatorname{diam}_{\|\cdot\|}(\mathcal{D})^2 L,
$$
where $\operatorname{diam}_{\|\cdot\|}(\mathcal{D}) \coloneqq \sup_{\mathbf{x},\mathbf{y} \in \mathcal{D}} \| \mathbf{x}-\mathbf{y} \|$.
\end{lemma}

\begin{lemma}{(\citet[Lemma A.7]{chen2022federated})}
\label{lemma:f_delta}
Given $n\in\mathbb{N}$, let $f^{(n)}(x) =  \sum_{i=1}^{n} \frac{x^{i-1}e^{-x}}{(i-1)!}$ for $x \in (0, +\infty)$. Then, $(f^{(n)})^{-1}(\delta)=(1+o(1)) \log(1/\delta)$ as $\delta \downarrow 0$, i.e., 
\begin{equation}
    \lim_{\delta \downarrow 0} \frac{(f^{(n)})^{-1}(\delta)}{\log(1/\delta)} = 1.
    \label{eq:f_delta}
\end{equation}
\end{lemma}

The below result demonstrates the concavity of the function $g^{(i,m)}_v(\cdot)\big|_{\Gamma^{(\eta)}}$ for all $(i,m)$ with $i \neq i_m^*(v)$.
\begin{lemma}
    \label{lemma:g_v_concave}
    Fix $v\in \mathbb{P}$, $\eta>0$, $m\in [M]$, and $i\in[K]$ such that $i \neq i^*_m(v)$. Then, $g^{(i,m)}_v(\cdot)\big|_{\Gamma^{(\eta)}}$ is a concave function.
\end{lemma}
\begin{proof}[Proof of Lemma~\ref{lemma:g_v_concave}]
It can be easily shown that the eigenvalues of the Hessian matrix of $g^{(i,m)}_v(\mathbf{\cdot})$ are 
\begin{itemize}
    \item $\dfrac{-\Delta_{i,m}^2(\omega_i^2+\omega_{i^*_m(v)}^2)}{(\omega_i+\omega_{i^*_m(v)})^3}$ with algebraic multiplicity $1$, and
    
    \item $0$ with algebraic multiplicity $K-1$.
\end{itemize}
Hence, the Hessian matrix of $g^{(i,m)}_v(\cdot)$ is negative semi-definite matrix for all $\omega \in \Gamma^{(\eta)}$, thus proving the desired result.
\end{proof}

With the above ingredients in place, we now prove that $C(v,\eta)$ is finite for all $\eta>0$.
\begin{lemma}
    \label{lemma:finite_cuvature}
    Fix $v \in \mathcal{P}$. For any $\eta>0$,
    \begin{equation}
        C(v,\eta) \le \max_{i \neq i^*_m(v)} \frac{2 \, \Delta_{i,m}^2(v) \, (1+\eta) \, K }{\eta}.
        \label{eq:C-v-eta-finite}
    \end{equation}
\end{lemma}
\begin{proof}[Proof of Lemma~\ref{lemma:finite_cuvature}]
In this proof, we will first prove that the curvature of $g^{(i,m)}_v(\cdot)\big|_{\Gamma^{(\eta)}}$ is finite for all $i\in[K]$ and $m\in [M]$. Subsequently, we will show that $C(v,\eta)$ is no greater than the maximum curvature of $g^{(i,m)}_v(\cdot)\big|_{\Gamma^{(\eta)}}$ computed over all $(i,m)$ pairs.

Note that for any $i \in [K]$, $m\in [M]$, $\imath \in [K]$ and $\jmath \in [K]$, we have $\forall \omega \in \Gamma^{(\eta)}$
\begin{equation}
    \frac{\partial }{\partial \omega_{\imath}  } g^{(i,m)} (\omega) = 
    \begin{cases}
        \dfrac{\Delta^2_{i,m}(v) \, \omega_{i^*_m(v)}^2}{2(\omega_{i^*_m(v)}+\omega_{i})^2}, & \text{ if } \imath = i, \\
        \\
        \dfrac{\Delta^2_{i,m}(v) \, \omega_{i}^2}{2(\omega_{i^*_m(v)}+\omega_i)^2}, & \text{ if } \imath = i^*_m(v), \\
        \\
        0,  & \text{ otherwise, } \\
    \end{cases}
    \label{eq:single_di}
\end{equation}
and
\begin{equation}
    \frac{\partial^2 }{\partial \omega_{\imath} \partial \omega_{\jmath} } g^{(i,m)} (\omega) = 
    \begin{cases}
        -\dfrac{\Delta^2_{i,m}(v) \, \omega_{i^*_m(v)}^2}{(\omega_{i^*_m(v)}+\omega_i)^3}, &   \text{ if } \imath=\jmath=i, \\
        \\
        -\dfrac{\Delta^2_{i,m}(v) \, \omega_{i}^2}{(\omega_{i^*_m(v)}+\omega_i)^3}, & \text{ if } \imath=\jmath=i^*_m(v), \\
        \\
        \dfrac{\Delta^2_{i,m}(v) \, \omega_{i}\omega_{i^*_m(v)}}{(\omega_{i^*_m(v)}+\omega_i)^3}, & \text{ if } \imath=i, \jmath=i^*_m(v)  \text{ or }  \imath=i^*_m(v), \jmath=i \\
        0, & \text{ otherwise.
        }
    \end{cases}
    \label{eq:double_di}
\end{equation}
Recall that $\omega \in \Gamma^{(\eta)}$ implies $\frac{\eta}{K(1+\eta)} \le \omega_i \le 1 $ for all $i \in [K]$. Using this fact along with~\eqref{eq:double_di}, we get that
$$
    \sup_{\omega \in \Gamma^{(\eta)}} \frac{\partial^2 }{\partial \omega_{\imath} \partial \omega_{\jmath} } g^{(i,m)} (\omega)  < +\infty \quad \forall i, \imath, \jmath \in [K], ~ m \in [M].
$$
In addition,~\eqref{eq:single_di} implies that the function $\nabla g^{(i,m)}: \Gamma^{(\eta)} \to \mathbb{R}^K$ is continuous. Combining the above facts, we see that $\nabla g^{(i,m)}: \Gamma^{(\eta)} \to \mathbb{R}$ is Lipschitz-continuous w.r.t.\ the  norm $\|\cdot\|_1$ with a finite Lipschitz constant $L$ satisfying
\begin{align}
    L & \le \sup_{\omega\in \Gamma^{(\eta)},\mathbf{d} \in \mathbb{R}^K} \frac{\| \mathbf{H}_{g^{(i,m)}}\left( \omega \right)  \mathbf{d} \|_1 } {\|\mathbf{d}\|_1} \nonumber \\
    & = \sup_{\omega\in \Gamma^{(\eta)}, \ \jmath\in[K]} {\| \mathbf{H}_{{g^{(i,m)}}}^{(\cdot,\jmath)}\left( \omega \right)   \|_1 } \nonumber  \\ 
    & \stackrel{(a)}{\le} \sup_{\omega\in \Gamma^{(\eta)}, \ \imath\in[K], \ \jmath\in[K]} {2\lvert \mathbf{H}^{(\imath,\jmath)}_{{g^{(i,m)}}}\left( \omega \right)   \rvert }  \nonumber \\ 
    & \stackrel{(b)}{\le} \frac{2\Delta_{i,m}^2(v)(1+\eta)K }{\eta},  
\end{align}
where $\mathbf{H}_{{g^{(i,m)}}}$ denotes the Hessian matrix of ${g^{(i,m)}}$ and $\mathbf{H}_{{g^{(i,m)}}}^{(\cdot,\jmath)}$ (resp.\  $\mathbf{H}_{{g^{(i,m)}}}^{(\imath,\jmath)}$) denotes its $\jmath$-th column (resp.\ its $(\imath, \jmath))$-th element), $(a)$ above follows the fact that the Hessian matrix of  $g^{(i,m)} (\cdot)\big|_{\Gamma^{(\eta)}}$ has at most two non-zero elements in each column, a fact that in turn follows from~\eqref{eq:double_di}, and $(b)$ above follows the fact that
$$
    \frac{\partial^2 }{\partial \omega_{\imath} \partial \omega_{\jmath} } g^{(i,m)} (\omega) \le \frac{\Delta_{i,m}^2(v)(1+\eta)K }{\eta} \quad \forall \omega \in \Gamma^{(\eta)}.
$$
Using Lemma~\ref{Lemma:jaggi} along with the fact that ${\rm diam}_{\|\cdot\|_1}(\Gamma^{(\eta)}) \leq 1$, we have for any $i\in[K]$ and $m\in[M]$
\begin{equation}
    C_{\rm cur}\left(g_v^{(i,m)}\big|_{\Gamma^{(\eta)}}\right) \leq \frac{2 \, \Delta_{i,m}^2(v) \, (1+\eta) \, K }{\eta}.
    \label{eq:cur_is_finite}
\end{equation}
In addition, from~\eqref{eq:def_C} and~\eqref{def:h}, we note by setting $\mathbf{z}=\omega+\gamma(\mathbf{y}-\omega)$ that
\begin{align}
    C(v,\eta) =  \sup_{\substack{\omega, \mathbf{y} \in \Gamma^{(\eta)}, \\ \gamma\in (0,1)}} \frac{2}{\gamma^2 } \bigg( \min_{(i,m): i\neq i^*_m(v)} g^{(i,m)}_{v}(\omega) + \langle  \nabla_{\omega} g^{(i,m)}_{v}(\omega), \mathbf{z}- \omega\rangle - g_v(\mathbf{z}) \bigg),
    \label{eq:re_C}
\end{align}
Note that for any $\mathbf{z} \in \Gamma^{(\eta)}$, there exist $i'\in[K]$ and $m'\in[M]$ with $i'\neq i^*_{m'}(v)$ such that $g_v(\mathbf{z}) = g^{(i',m')}_v(\mathbf{z})$. Therefore,
\begin{align}
    & \min_{(i,m): i\neq i^*_m(v)} g^{(i,m)}_{v}(\omega) + \langle  \nabla_{\omega} g^{(i,m)}_{v}(\omega), \mathbf{z}- \omega\rangle - g_v(\mathbf{z}) \nonumber \\
    & = \min_{(i,m): i\neq i^*_m(v)} g^{(i,m)}_{v}(\omega) + \langle  \nabla_{\omega} g^{(i,m)}_{v}(\omega), \mathbf{z}- \omega\rangle - g^{(i',m')}_v(\mathbf{z}) \nonumber  \\
    & \le  g^{(i',m')}_{v}(\omega) + \langle  \nabla_{\omega} g^{(i',m')}_{v}(\omega), \mathbf{z}- \omega\rangle - g^{(i',m')}_v(\mathbf{z}) \nonumber \\
    & \le  \max_{(i,m): i\neq i^*_m(v)} \bigg[ g^{(i,m)}_{v}(\omega) + \langle  \nabla_{\omega} g^{(i,m)}_{v}(\omega), \mathbf{z}- \omega\rangle - g^{(i,m)}_v(\mathbf{z}) \bigg].
    \label{eq:less_than_max_g}
\end{align}
Combining~\eqref{eq:re_C} and~\eqref{eq:less_than_max_g},  and setting $\mathbf{z}=\omega+\gamma(\mathbf{y}-\omega)$ throughout, we get that
\begin{align}
    C(v,\eta) 
    & \le \sup_{\substack{\omega, \mathbf{y}\in \Gamma^{(\eta)}, \\ \gamma\in (0,1)}} \ \max_{(i,m): i\neq i^*_m(v)} \bigg[ g^{(i,m)}_{v}(\omega) + \langle  \nabla_{\omega} g^{(i,m)}_{v}(\omega), \mathbf{z}- \omega\rangle - g^{(i,m)}_v(\mathbf{z}) \bigg],  \nonumber \\
    & =  \max_{(i,m): i\neq i^*_m(v)} \ \sup_{\substack{\omega, \mathbf{y}\in \Gamma^{(\eta)}, \\ \gamma\in (0,1)}}  \bigg[ g^{(i,m)}_{v}(\omega) + \langle  \nabla_{\omega} g^{(i,m)}_{v}(\omega), \mathbf{z}- \omega\rangle - g^{(i,m)}_v(\mathbf{z}) \bigg],  \nonumber \\
    &\stackrel{(a)}{=} \max_{(i,m): i\neq i^*_m(v)} C_{\rm cur}\left(g_v^{(i,m)}\big|_{\Gamma^{(\eta)}}\right),
    \label{eq:C_less_than_max_g}
\end{align}
where (a) follows from the definition in \eqref{eq:def_Cur}.
Finally, combining~\eqref{eq:cur_is_finite} and~\eqref{eq:C_less_than_max_g}, we arrive at the desired result.
\end{proof}

\begin{lemma}
    \label{lemma:finite_emprical_cuvature}
    Fix $\eta>0$. Consider the sequence $\{v_{l_t}: t \geq K\}$ generated by $\textsc{MO-BAI}$. Then, 
    $$
    \limsup_{t \rightarrow \infty} C(\widehat{v}_{l_t},\eta) < + \infty \quad \text{almost surely}.
    $$
\end{lemma}
\begin{proof}[Proof of Lemma~\ref{lemma:finite_emprical_cuvature}]
Fix $v \in \mathcal{P}$. By the strong law of large numbers, we have for all $ i\in[K]$ and $m\in[M]$
$$
\lim_{t \rightarrow +\infty} \Delta_{i,m}(\widehat{v}_{l_t}) = \Delta_{i,m}(v) \quad \text{almost surely}
$$
under the instance $v$, which by Lemma~\ref{lemma:finite_cuvature} implies that
$$
\limsup_{t \rightarrow \infty} C(\widehat{v}_{l_t},\eta) < + \infty \quad \text{almost surely}.
$$
This completes the proof.
\end{proof}


\begin{lemma} 
    \label{lemma:g_error_vanish}
    Fix $v\in \mathcal{P}$ and $\eta>0$. Under $\textsc{MO-BAI}$, we have
    $$
    \lim_{t\rightarrow +\infty} \  \max_{\omega \in \Gamma^{(\eta)}} \  \lvert g_v(\omega)-g_{\widehat{v}_{l_t}}(\omega) \rvert = 0 \quad \text{almost surely}.
    $$
\end{lemma}
\begin{proof}[Proof of Lemma~\ref{lemma:g_error_vanish}]
By the strong law of large numbers and the fact that under $v$ the best arm of each objective is unique, we have for all $m\in[M]$
\begin{equation}
    \lim_{t \rightarrow +\infty} i^*_m(\widehat{v}_{l_t}) = i_m^*(v) \quad \text{ almost surely}, 
    \label{eq:i_star_converge}
\end{equation}
and for all $i\in [K]$ and $m\in[M]$
\begin{equation}
    \lim_{t \rightarrow +\infty} \Delta_{i,m}(\widehat{v}_{l_t}) = \Delta_{i,m}(v) \quad \text{ almost surely.}
    \label{eq:delta_converge}
\end{equation}
For $m\in[M]$, let $T_m$ be the smallest integer such that 
\begin{equation}    
    i^*_m(\widehat{v}_{l_t}) = i^*(v) \quad \forall t \ge T_m.
    \label{eq:def_t_m}
\end{equation}
Then, \eqref{eq:i_star_converge} implies that  $T_m < +\infty$ almost surely.
In addition, by the definition of $g^{(i,m)}_v(\cdot)$ in~\eqref{eq:def_g_i_m}, we have $\forall t\ge T_m$ and $\omega \in \Gamma^{(\eta)}$ that
\begin{align}
    \lvert g^{(i,m)}_v (\omega) - g^{(i,m)}_{\widehat{v}_{l_t}}(\omega) \rvert & = \frac{\omega_i \, \omega_{i_m^*(v)}}{\omega_i + \omega_{i_m^*(v)}} \lvert  \Delta_{i,m}(\widehat{v}_{l_t}) - \Delta_{i,m}(v) \rvert \nonumber \\
    & \le \lvert \Delta_{i,m}(\widehat{v}_{l_t}) - \Delta_{i,m}(v) \rvert,
\end{align}
which implies that $\forall t\ge \max_{m\in [M]}{T_m}$,
\begin{equation}
    \lvert g_v (\omega) - g_{\widehat{v}_{l_t}}(\omega) \rvert \le \max_{(i,m) \in [K]\times [M]} \lvert \Delta_{i,m}(\widehat{v}_{l_t}) - \Delta_{i,m}(v) \rvert. 
    \label{eq:g_omega_gap}
\end{equation}
Notice that the right-hand side of~\eqref{eq:g_omega_gap} does not depend on $\omega$. Hence, combing~\eqref{eq:delta_converge} and~\eqref{eq:g_omega_gap} along with the almost sure finiteness of $T_m$ for each $m \in [M]$, we arrive at the desired result.
\end{proof}

Next, we show that the empirical proportion of arms pulls under $\textsc{MO-BAI}$ converges to the oracle weight in the long run. 

\begin{lemma} 
\label{lemma:g_to_c_star} 
Fix $v \in \mathcal{P}$ and $\eta>0$. For all $t_1,t_2 \in \mathbb{N}$ with $t_2>t_1 > K$, under $\textsc{MO-BAI}$, we have
\begin{align}
   \lvert \tilde{c}(v,\eta)^{-1}-g_v(\widehat{\omega}_{\cdot,t_2}) \rvert \le  \frac{t_1}{t_2} \, \tilde{c}(v,\eta)^{-1}  + 11 \, \epsilon_{t_1}(v) +  \frac{2 \, \log(t_2) \, \overline{C}_{t_1}(\eta)}{t_2},
    \label{eq:relation-between-tilde-c-and-g-v}
\end{align} 
where for any time step $t$:
\begin{itemize}

    \item The quantity $\epsilon_t(v)$ is defined as
    \begin{equation}
        \epsilon_t(v) \coloneqq \sup_{t' \ge l_t} \left( \sup_{\omega \in \Gamma^{(\eta)}} \big\lvert g_v(\omega)-g_{\widehat{v}_{t'}}(\omega) \big\rvert  \right).
        \label{def:epsilon_t}
    \end{equation}

    \item The quantity $\overline{C}_t(\eta)$ is defined as
    \begin{equation}
        \overline{C}_t(\eta) \coloneqq \sup_{t' \ge l_t} C(\widehat{v}_{t'},\eta).
        \label{def:overline_c}
    \end{equation}
\end{itemize} 
Consequently, letting $t_1= \lceil \sqrt{t_2} \rceil$, we have
\begin{align}
    \tilde{c}(v,\eta)^{-1}-g_v(\widehat{\omega}_{\cdot,t_2}) \le  \frac{2}{\sqrt{t_2}}\tilde{c}(v,\eta)^{-1}  + 11\epsilon_{\lceil \sqrt{t_2} \rceil}(v)+  \frac{2 \, \log(t_2) \, \overline{C}_{\lceil \sqrt{t_2} \rceil}(\eta)}{t_2}.
    \label{eq:proof-of-upper-bound-1}
\end{align}
\end{lemma}
\begin{proof}[Proof of Lemma~\ref{lemma:g_to_c_star}]

Fix $t_1,t_2$ with $K<t_1 < t_2$.

\textbf{Step 1: Bound on the empirical instance $\widehat{v}_{l_t}$}. 

Fix $t\in \{t_1+1, \ldots, t_2 \}$. Let $\widehat{v}_{l_{t}} \in \mathcal{P}$ denote the instance in which the mean of arm $i$ corresponding to objective $m$ is given by $\widehat{\mu}_{i,m}(l_{t})$. Making the substitutions $v \leftarrow \widehat{v}_{l_{t}}$, $\omega \leftarrow \widehat{\omega}_{\cdot, t-1}$, and $\mathbf{s} \leftarrow \widehat{\omega}_{\cdot, t}$ in~\eqref{def:h}, we may deduce the following: there exists $(i', m') \in [K] \times [M]$, $i' \neq i_{m'}^*(\widehat{v}_{l_{t}})$, such that
\begin{equation}
    g^{(i',m')}_{\widehat{v}_{l_{t}}}(\widehat{\omega}_{\cdot,t-1}) + \langle  \nabla_{\omega} \, g^{(i',m')}_{\widehat{v}_{l_{t}}}(\widehat{\omega}_{\cdot,t-1}), \  \widehat{\omega}_{\cdot, t}  - \widehat{\omega}_{\cdot, t-1} \rangle = h_{\widehat{v}_{l_{t}}}(\widehat{\omega}_{\cdot,t-1}, \widehat{\omega}_{\cdot,t}).
    \label{eq:i_prime_m_prime}
\end{equation}
Recall that $\mathbf{s}_{t}=\arg\max_{\mathbf{s} \in \Gamma^{(\eta)}} h_{\widehat{v}_{l_{t}}}(\widehat{\omega}_{\cdot,t-1}, \mathbf{s})$. We then have from~\eqref{def:h} that
\begin{align}
    &g^{(i',m')}_{\widehat{v}_{l_{t}}}(\widehat{\omega}_{\cdot,t-1}) + \langle  \nabla_{\omega} \, g^{(i',m')}_{\widehat{v}_{l_{t}}}(\widehat{\omega}_{\cdot,t-1}), \ \mathbf{s}_{t}  - \widehat{\omega}_{\cdot, t-1} \rangle \nonumber\\
    &\hspace{3cm} \ge  h_{\widehat{v}_{l_{t}}}(\widehat{\omega}_{\cdot, t-1}, \mathbf{s}_t) \nonumber\\
    &\hspace{3cm} \stackrel{(a)}{\ge} \tilde{c}(\widehat{v}_{l_{t}},\eta)^{-1},
    \label{eq:larger_than_max}
\end{align}
where (a) follows from Lemma~\ref{lemma:g_v_concave}. 
Let
\begin{equation}
\label{def:U_t}
    U_{t} \coloneqq  g^{(i',m')}_{\widehat{v}_{l_{t}}}(\widehat{\omega}_{\cdot,t-1}) + \langle  \nabla_{\omega} \,  g^{(i',m')}_{\widehat{v}_{l_{t}}}(\widehat{\omega}_{\cdot,t-1}), \ \mathbf{s}_{t}  - \widehat{\omega}_{\cdot, t-1} \rangle - g_{\widehat{v}_{l_{t}}}(\widehat{\omega}_{\cdot,t-1}),
\end{equation}
From~\eqref{eq:larger_than_max}, we have
\begin{equation}
    U_{t} \ge \tilde{c}(\widehat{v}_{l_{t}},\eta)^{-1} - g_{\widehat{v}_{l_{t}}}(\widehat{\omega}_{\cdot,t-1}).
\end{equation}
We then note that
\begin{align}
    & h_{\widehat{v}_{l_{t}}}(\widehat{\omega}_{\cdot,t-1}, \widehat{\omega}_{\cdot,t}) -  g_{\widehat{v}_{l_{t}}}(\widehat{\omega}_{\cdot,t-1}) \nonumber\\
    & \stackrel{(a)}{=}  g^{(i',m')}_{\widehat{v}_{l_{t}}}(\widehat{\omega}_{\cdot,t-1}) + \langle  \nabla_{\omega} \, g^{(i',m')}_{\widehat{v}_{l_{t}}}(\widehat{\omega}_{\cdot,t-1}), \  \widehat{\omega}_{\cdot, t}  - \widehat{\omega}_{\cdot, t-1} \rangle - g_{\widehat{v}_{l_{t}}}(\widehat{\omega}_{\cdot,t-1}) \nonumber \\
    &\stackrel{(b)}{=}    \langle  \nabla_{\omega} g^{(i',m')}_{\widehat{v}_{l_{t}}}(\widehat{\omega}_{\cdot,t-1}), \   \frac{1}{t}(\mathbf{s}_t  - \widehat{\omega}_{\cdot, t-1}) \rangle  +g^{(i',m')}_{\widehat{v}_{l_{t}}}(\widehat{\omega}_{\cdot,t-1})  - g_{\widehat{v}_{l_{t}}}(\widehat{\omega}_{\cdot,t-1}) \nonumber \\
    & \ge \frac{1}{t} \left( \langle  \nabla_{\omega} g^{(i',m')}_{\widehat{v}_{l_{t}}}(\widehat{\omega}_{\cdot,t-1}),   \ \mathbf{s}_{t}  - \widehat{\omega}_{\cdot, t-1} \rangle  +g^{(i',m')}_{\widehat{v}_{l_{t}}}(\widehat{\omega}_{\cdot,t-1})  - g_{\widehat{v}_{l_{t}}} (\widehat{\omega}_{\cdot,t-1}) \right) \nonumber \\
    &= \frac{U_{t}}{t} \nonumber\\
    &\ge \frac{1}{t} \left(\tilde{c}(\widehat{v}_{l_t},\eta)^{-1} - g_{\widehat{v}_{l_t}}(\widehat{\omega}_{\cdot,t-1}) \right),
    \label{eq:bound_on_empirical}
\end{align}
where $(a)$ above follows from~\eqref{eq:i_prime_m_prime}, and $(b)$ above follows from the construction of $\widehat{\omega}_{\cdot, t}$ (see Algorithm~\ref{alg:mo-bai}).

\textbf{Step 2: Bound on the instance $v$.}

For $t\in \{t_1+1, \ldots, t_2\}$, by the definition of $C(\cdot)$ in \eqref{eq:def_C},  we have
\begin{equation}
    g_{\widehat{v}_{l_t}}(\widehat{\omega}_{\cdot,t}) \ge h_{\widehat{v}_{l_t}}(\widehat{\omega}_{\cdot,t-1}, \widehat{\omega}_{\cdot,t}) - \frac{2 \, \overline{C}_t(\eta)}{{t}^2}.
    \label{eq:step_in_empirical_instance}
\end{equation}
Then, combining~\eqref{eq:bound_on_empirical} and~\eqref{eq:step_in_empirical_instance}, we get that for all $t\in \{t_1+1, \ldots, t_2\}$,
\begin{equation}
    g_{\widehat{v}_{l_t}}(\widehat{\omega}_{\cdot,t}) \ge   g_{\widehat{v}_{l_t}}(\widehat{\omega}_{\cdot,t-1})  + \left(\frac{1}{t} \left(\tilde{c}({\widehat{v}_{l_t}},\eta)^{-1} - g_{{\widehat{v}_{l_t}}}(\widehat{\omega}_{\cdot,t-1}) \right) - \frac{2 \, \overline{C}_t(\eta)}{t^2} \right).
    \label{eq:first_things_in_converge}
\end{equation}
It then follows from \eqref{eq:first_things_in_converge} that
\begin{align}
    & \tilde{c}(\widehat{v}_{l_t},\eta)^{-1} - g_{\widehat{v}_{l_t}}(\widehat{\omega}_{\cdot,t}) \nonumber\\ 
    &\le  \tilde{c}(\widehat{v}_{l_t},\eta)^{-1}- g_{\widehat{v}_{l_t}}(\widehat{\omega}_{\cdot,t-1})  -  \left( \frac{1}{t} \left(\tilde{c}(\widehat{v}_{l_t},\eta)^{-1} - g_{\widehat{v}_{l_t}}(\widehat{\omega}_{\cdot,t-1}) \right) - \frac{2 \, \overline{C}_t(\eta)}{t^2} \right) \nonumber\\
    &= \frac{t-1}{t} \left(\tilde{c}(\widehat{v}_{l_t},\eta)^{-1} - g_{\widehat{v}_{l_t}}(\widehat{\omega}_{\cdot,t-1}) \right) + \frac{2 \, \overline{C}_t(\eta)}{t^2}.
    \label{eq:iteration_0}
\end{align}

Let
\begin{equation}
    G_t \coloneqq \tilde{c}(\widehat{v}_{l_t},\eta)^{-1} - g_{\widehat{v}_{l_t}}(\widehat{\omega}_{\cdot,t}).
    \label{def:g_t}
\end{equation}
Observe that $l_{t-1} = l_t$ if and only if $t \notin \{2^i \given[\big] i\in \mathbb{N} \}$. Using the definition of $\epsilon_t$ from the statement of the lemma, and combining~\eqref{eq:iteration_0} and~\eqref{def:g_t}, we get that for all $t\in \{t_1+1,\ldots,t_2\}$,
\begin{equation}
    G_{t} \le  
    \begin{cases}
        \dfrac{t-1}{t} \, G_{t-1}  + \dfrac{2 \, \overline{C}_t(\eta)}{t^2} + 4 \, \epsilon_{t_1}(v),  &\text{if  } t \in \{2^i \given[\big] i\in \mathbb{N} \},  \\
        \\
        \dfrac{t-1}{t} \, G_{t-1}  + \dfrac{2 \, \overline{C}_t(\eta)}{t^2},  &\text{if  } t \notin \{2^i \given[\big] i\in \mathbb{N}  \}. 
    \end{cases}
    \label{eq:g_t_iteration}
\end{equation}
The first line above follows by noting that for $t \in \{2^i \given[\big] i\in \mathbb{N} \}$,
\begin{align}
    &\tilde{c}(\widehat{v}_{l_t},\eta)^{-1} - g_{\widehat{v}_{l_t}}(\widehat{\omega}_{\cdot,t-1}) - G_{t-1}\nonumber\\*
    &= \tilde{c}(\widehat{v}_{l_t},\eta)^{-1} - g_{\widehat{v}_{l_t}}(\widehat{\omega}_{\cdot,t-1}) - \tilde{c}(\widehat{v}_{l_{t-1}},\eta)^{-1} + g_{\widehat{v}_{l_{t-1}}}(\widehat{\omega}_{\cdot,t-1}) \nonumber\\
    &= \tilde{c}(\widehat{v}_{l_t},\eta)^{-1} - \tilde{c}(\widehat{v}_{l_{t-1}},\eta)^{-1} + g_{\widehat{v}_{l_{t-1}}}(\widehat{\omega}_{\cdot,t-1}) - g_{\widehat{v}_{l_t}}(\widehat{\omega}_{\cdot,t-1}) \nonumber\\
    &\leq \sup_{\omega \in \Gamma^{(\eta)}} g_{\widehat{v}_{l_t}}(\omega) - \sup_{\omega \in \Gamma^{(\eta)}} g_{\widehat{v}_{l_{t-1}}}(\omega) + g_{\widehat{v}_{l_{t-1}}}(\omega) - g_{\widehat{v}_{l_t}}(\omega) \nonumber\\
    &\leq \sup_{\omega \in \Gamma^{(\eta)}} \big[g_{\widehat{v}_{l_t}}(\omega) - g_{\widehat{v}_{l_{t-1}}}(\omega)\big] + |g_{\widehat{v}_{l_{t-1}}}(\omega) - g_{\widehat{v}_{l_t}}(\omega)| \nonumber\\
    &\leq 2 \sup_{\omega \in \Gamma^{(\eta)}} |g_{\widehat{v}_{l_{t-1}}}(\omega) - g_{\widehat{v}_{l_t}}(\omega)| \nonumber\\
    &\leq 2 \bigg[\sup_{\omega \in \Gamma^{(\eta)}} |g_{v}(\omega) - g_{\widehat{v}_{l_t}}(\omega)| + \sup_{\omega \in \Gamma^{(\eta)}} |g_{v}(\omega) - g_{\widehat{v}_{l_{t-1}}}(\omega)|\bigg] \nonumber\\
    &\leq 2 \bigg[\sup_{t' \geq l_t} \ \sup_{\omega \in \Gamma^{(\eta)}} |g_{v}(\omega) - g_{\widehat{v}_{t'}}(\omega)| + \sup_{t' \geq l_{t-1}} \ \sup_{\omega \in \Gamma^{(\eta)}} |g_{v}(\omega) - g_{\widehat{v}_{t'}}(\omega)|\bigg] \nonumber\\
    &\leq 4 \, \epsilon_{t-1}(v) \nonumber\\
    &\leq 4 \, \epsilon_{t_1}(v).
    \label{eq:line-1-proof}
\end{align}
The penultimate step above follows by noting that $l_{t-1} < l_t$ for $t \in \{2^i \given[\big] i\in \mathbb{N} \}$, and the last step above follows by using the fact that $\epsilon_{t-1} \leq \epsilon_{t_1}$ for all $t \in \{t_1+1, \ldots, t_2\}$.
Using the mathematical induction formula given in the following Lemma~\ref{lemma:induction_G_t}, it follows from~\eqref{eq:g_t_iteration} that for all $K < t_1 < t_2$,
\begin{equation}
    G_{t_2} \le \frac{t_1}{t_2} \, G_{t_1} + 4 \, \epsilon_{t_1}(v) \, \left(\sum_{t=t_1}^{t_2} \frac{t}{t_2} \, \mathbf{1}_{\{ t \in \{2^i | i\in \mathbb{N} \} \} } \right)+  \frac{2 \, \overline{C}_{t_1}(\eta)}{t_2} \left(\sum_{j=t_1}^{t_2} \frac{1}{j}\right). \label{eq:bound_g2_1}
\end{equation}
Note that 
\begin{align}
    \sum_{t=t_1}^{t_2} \frac{t}{t_2} \, \mathbf{1}_{\{ t \in \{2^i | i\in \mathbb{N} \} \} }  & \le  \sum_{i=0}^{\infty} 2^{-i} = 2.  \label{eq:bound_g2_2}
\end{align}
In addition, we have
\begin{align}
\sum_{j=t_1}^{t_2} \frac{1}{j} & \stackrel{(a)}{\le} \sum_{j=2}^{t_2} \frac{1}{j} \nonumber \\
& \le \int_{1} ^{t_2} \frac{1}{x} \, {\rm d}x  \\
& = \log(t_2), \label{eq:bound_g2_3}
\end{align}
where $(a)$ above follows from the fact that $t_1 \ge K \ge 2$. Combining~\eqref{eq:bound_g2_1},~\eqref{eq:bound_g2_2}, and~\eqref{eq:bound_g2_3}, we have
\begin{align}
    G_{t_2} 
    & \le \frac{t_1}{t_2} \, G_{t_1} + 8 \, \epsilon_{t_1}(v) +  \frac{2 \, \log(t_2) \, \overline{C}_{t_1}(\eta)}{t_2} \nonumber \\
    & \stackrel{(a)}{\le} \frac{t_1}{t_2} \, \tilde{c}(\widehat{v}_{l_{t_1}}, \eta)^{-1} + 8 \, \epsilon_{t_1}(v)+  \frac{2 \, \log(t_2) \, \overline{C}_{t_1}(\eta)}{t_2} \nonumber \\
    & \stackrel{(b)}{\le} \frac{t_1}{t_2} \, \left(\tilde{c}(v,\eta)^{-1} + \epsilon_{t_1}(v)\right) + 8 \, \epsilon_{t_1}(v) +  \frac{2 \, \log(t_2) \, \overline{C}_{t_1}(\eta)}{t_2}, \label{eq:gt_to_v_1}
\end{align}
where $(a)$ above follows by noting that $G_{t_1}  = \tilde{c}(\widehat{v}_{l_{t_1}},\eta)^{-1} - g_{\widehat{v}_{l_{t_1}}}(\widehat{\omega}_{\cdot,t_1}) \le \tilde{c}(\widehat{v}_{l_{t_1}}, \eta)^{-1}$, and $(b)$ above follows from~\eqref{def:epsilon_t}.
Additionally, we note that \eqref{def:epsilon_t} implies 
\begin{align}
    \tilde{c}(v,\eta)^{-1}-g_v(\widehat{\omega}_{\cdot,t_2}) 
    &= \tilde{c}(v,\eta)^{-1}-g_v(\widehat{\omega}_{\cdot,t_2}) - G_{t_2} + G_{t_2} \nonumber\\
    &= \tilde{c}(v,\eta)^{-1}-g_v(\widehat{\omega}_{\cdot,t_2}) - \big(\tilde{c}(\widehat{v}_{l_{t_2}},\eta)^{-1} - g_{\widehat{v}_{l_{t_2}}}(\widehat{\omega}_{\cdot,t_2})\big) + G_{t_2} \nonumber\\
    &= \tilde{c}(v,\eta)^{-1} - \tilde{c}(\widehat{v}_{l_{t_2}},\eta)^{-1} + g_{\widehat{v}_{l_{t_2}}}(\widehat{\omega}_{\cdot,t_2}) - g_v(\widehat{\omega}_{\cdot,t_2}) + G_{t_2} \nonumber\\
    &= \sup_{\omega \in \Gamma^{(\eta)}} g_{v}(\omega) - \sup_{\omega \in \Gamma^{(\eta)}} g_{\widehat{v}_{l_{t_2}}}(\omega) + \sup_{\omega \in \Gamma^{(\eta)}}|g_v(\omega) - g_{\widehat{v}_{l_{t_2}}}(\omega)| + G_{t_2} \nonumber\\
    &\leq 2 \, \sup_{\omega \in \Gamma^{(\eta)}} |g_v(\omega) - g_{\widehat{v}_{l_{t_2}}}(\omega)| + G_{t_2} \nonumber\\
    &\leq G_{t_2} + 2 \, \epsilon_{t_2}(v) \nonumber\\
    &\le G_{t_2} + 2 \, \epsilon_{t_1}(v). 
    \label{eq:gt_to_v_2}
\end{align}
Finally, combining~\eqref{eq:gt_to_v_1} and~\eqref{eq:gt_to_v_2}, we arrive at the desired result.
\end{proof}


\begin{lemma}
\label{lemma:induction_G_t}
Fix $v \in \mathcal{P}$, $\eta>0$, and $t_1,t_2 \in \mathbb{N}$ such that $t_1<t_2$. Then, under $\textsc{MO-BAI}$,
\begin{equation}
\label{eq:lemma_gt}
    G_{t_2} \le \frac{t_1}{t_2} \, G_{t_1} + 4 \, \epsilon_{t_1}(v) \, \sum_{t=t_1}^{t_2} \frac{t}{t_2} \, \mathbf{1}_{\{ t \in \{2^i | i\in \mathbb{N} \} \} } +  \frac{2 \, \overline{C}_{t_1}(\eta)}{t_2} \, \sum_{j=t_1}^{t_2} \frac{1}{j},
\end{equation}
where $G_t$, $\epsilon_t(v)$,  and $\overline{C}_{t}(\eta)$ are as defined in~\eqref{def:g_t},~\eqref{def:overline_c} and~\eqref{def:epsilon_t} respectively for $t\in \mathbb{N}$.
\end{lemma}
\begin{proof}[Proof of Lemma~\ref{lemma:induction_G_t}]
Using the principle of mathematical induction, we shall demonstrate that for any $t' \in \{t_1, \ldots, t_2\}$,
\begin{equation}
    G_{t'} \le \frac{t_1}{t'} \, G_{t_1} + 4 \, \epsilon_{t_1}(v) \, \sum_{t=t_1}^{t'} \frac{t}{t'} \, \mathbf{1}_{\{ t \in \{2^i | i\in \mathbb{N} \} \} } +  \frac{2 \, \overline{C}_{t_1}(\eta)}{t'} \, \sum_{j=t_1}^{t'} \frac{1}{j}.
    \label{eq:to_induction}
\end{equation}
\textbf{Base case:} For $t'=t_1$, we can immediately verify that~\eqref{eq:to_induction} holds.

\textbf{Induction Step:} Suppose that~\eqref{eq:to_induction} holds for $t'=s-1$ for some $s \in \{t_1+1, \ldots, t_2\}$. Then, ~\eqref{eq:g_t_iteration} implies that
\begin{align}
    G_{s} 
    &\le \frac{s-1}{s} \, G_{s-1} + \frac{2 \, \overline{C}_s(\eta)}{s^2} + \mathbf{1}_{\{ s \in \{2^i | i\in \mathbb{N} \} \}}  \, 4 \epsilon_s(v) \nonumber \\
    &\stackrel{(a)}{\le} \frac{s-1}{s} \, G_{s-1} + \frac{2 \, \overline{C}_s(\eta)}{s^2} + \mathbf{1}_{\{ s \in \{2^i | i\in \mathbb{N} \} \}} \, 4 \epsilon_{t_1}(v)  \nonumber\\
     &\stackrel{(b)}{\le} \frac{s-1}{s} \, \left(\frac{t_1}{s-1} \, G_{t_1} + 4 \, \epsilon_{t_1}(v) \, \sum_{t=t_1}^{s-1} \frac{t}{s-1} \mathbf{1}_{\{ t \in \{2^i | i\in \mathbb{N} \} \} } +   \frac{2 \, \overline{C}_{t_1}(\eta)}{s-1} \, \sum_{j=t_1}^{s-1} \frac{1}{j} \right) \nonumber\\*
     &\qquad\qquad+ \frac{2 \, \overline{C}_s(\eta)}{s^2} + \mathbf{1}_{\{ s \in \{2^i | i\in \mathbb{N} \} \}} \, 4  \, \epsilon_{t_1}(v)   \nonumber \\
     & = \frac{t_1}{s} \, G_{t_1} + 4 \, \epsilon_{t_1}(v) \, \sum_{t=t_1}^{s} \frac{t}{s} \,  \mathbf{1}_{\{ t \in \{2^i | i\in \mathbb{N} \} \} } +  \frac{2 \, \overline{C}_{t_1}(\eta)}{s} \, \sum_{j=t_1}^{s} \frac{1}{j},
\end{align}
where $(a)$ above follows from the fact that $\epsilon_s(v) < \epsilon_{t_1}(v)$ for $s>t_1$, and $(b)$ above follows from the induction hypothesis. We have thus demonstrated that \eqref{eq:to_induction} holds for $t'=s$, thereby the desired proof.
\end{proof}


\begin{lemma}
\label{lemma:Tlast}
Fix $v \in \mathcal{P}$ and $\eta>0$. Under $\textsc{MO-BAI}$, for any $\epsilon \in \big(0, \tilde{c}(v,\eta)^{-1} \big)$, there exists $\delta_{\rm thres}(v,\eta,\epsilon) > 0$ such that for all $\delta \in (0, \delta_{\rm thres}(v,\eta, \epsilon))$,
\begin{equation}
    t \tilde{c}(v,\eta)^{-1} > \beta(t,\delta) + \epsilon t \quad \forall t \ge T_{\rm thres}(v,\eta, \epsilon, \delta),
    \label{eq:laststoptime}
\end{equation}
where $T_{\rm thres}(v,\eta, \epsilon, \delta)$ is defined as
\begin{equation}
    T_{\rm thres}(v,\eta, \epsilon, \delta) \coloneqq \frac{f^{-1}(\delta)}{\tilde{c}(v,\eta)^{-1} - \epsilon} + \frac{MK}{\tilde{c}(v,\eta)^{-1} - \epsilon} \, \log\left( \left(\frac{2f^{-1}(\delta)}{\tilde{c}(v,\eta)^{-1} - \epsilon}\right)^2 + \frac{2f^{-1}(\delta)}{\tilde{c}(v,\eta)^{-1} - \epsilon}  \right) +1.
    \label{eq:denotetlast}
\end{equation}
In \eqref{eq:denotetlast}, $f^{-1}(\delta)$ is as defined in Section~\ref{subsec:stopping-rule}.
\end{lemma}
\begin{proof}[Proof of Lemma~\ref{lemma:Tlast}]
Fix $\eta>0$, $v \in \mathcal{P}$ and $\epsilon \in \left(0, \tilde{c}(v,\eta)^{-1} \right)$ arbitrarily.
Recall the relation $\beta(t,\delta) = MK\log(t^2+t) + f^{-1}(\delta)$ for the threshold $\beta(t, \delta)$ employed by $\textsc{MO-BAI}$. For any $\bar{t}>0$ and $\delta>0$, let
\begin{equation}
    H(\bar{t},\delta) \coloneqq \mathbf{1} \big\{\forall t \ge \bar{t},\  t \, \tilde{c}(v,\eta)^{-1} > MK \, \log(\bar{t} ^2+\bar{t}) + f^{-1}(\delta) + \epsilon \, t \big\}.
    \label{eq:repeat_last_stoptime}
\end{equation}
To complete the proof, it suffices to show that $H(T_{\rm thres}(v,\eta, \epsilon, \delta),\delta)=1$ for all sufficiently small values of $\delta$. Note that the following is a sufficient condition for $H(\bar{t},\delta)=1$: 
\begin{align}
    \begin{cases}
        \bar{t} \, \tilde{c}(v,\eta)^{-1} > MK \, \log(\bar{t}^2+\bar{t}) + f^{-1}(\delta) + \epsilon \, \bar{t}, \\ \\
        \forall t'\ge \bar{t}, \quad \dfrac{{\rm d} }{{\rm d}t} \, t \tilde{c}(v,\eta )^{-1} \bigg|_{t=t'}  \ge  \dfrac{{\rm d} }{{\rm d}t} \, \big[MK\log(t ^2+t) + f^{-1}(\delta) + \epsilon t \big] \bigg|_{t=t'}.
    \end{cases}
    \label{eq:sufficiency1}
\end{align}
By rearranging the inequalities of~\eqref{eq:sufficiency1}, we get that the following is a sufficient condition for $H(\bar{t}, \delta)=1$:
\begin{align}
    \begin{cases}
        \bar{t}  > \dfrac{MK \, \log(\bar{t}^2+\bar{t}) + f^{-1}(\delta)}{\tilde{c}(v,\eta)^{-1}-\epsilon},  \\
        \\
        \bar{t} \ge  \max \bigg\{\dfrac{3}{\tilde{c}(v,\eta)^{-1} - \epsilon}, \ 1 \bigg\}.
    \end{cases}
    \label{eq:sufficiency2}
\end{align}
Noting that appending an extra condition on $\bar{t}$ to the conditions in \eqref{eq:sufficiency2} does not affect the sufficiency of the conditions in \eqref{eq:sufficiency2} for $H(\bar{t}, \delta)=1$, we get that the following set of conditions is sufficient for $H(\bar{t},\delta)=1$:
\begin{align}
    \begin{cases}
        \bar{t}  > \dfrac{MK \, \log(\bar{t}^2+\bar{t}) + f^{-1}(\delta)}{\tilde{c}(v,\eta)^{-1}-\epsilon},  \\
        \\
        \bar{t} \ge  \max \bigg\{\dfrac{3}{\tilde{c}(v,\eta)^{-1} - \epsilon}, \ 1 \bigg\}, \\ 
        \\
        \bar{t} \le \dfrac{2 \, f^{-1}(\delta)}{\tilde{c}(v,\eta)^{-1} - \epsilon}. \\
        \end{cases}
    \label{eq:sufficiency3}
\end{align}
Note that for $\bar{t} \le \frac{2f^{-1}(\delta)}{\tilde{c}(v,\eta)^{-1} - \epsilon}$, we have 
$$
    \log\left(\left(\frac{2f^{-1}(\delta)}{\tilde{c}(v,\eta)^{-1} - \epsilon}\right)^2 + \frac{2f^{-1}(\delta)}{\tilde{c}(v,\eta)^{-1} - \epsilon}  \right)> \log(\bar{t}^2+\bar{t}).
$$ 
Continuing with~\eqref{eq:sufficiency3}, we get that the following set of conditions is sufficient to guarantee that $H(\bar{t},\delta)=1$:
\begin{align}
    \begin{cases}
        \bar{t} > \dfrac{f^{-1}(\delta)}{\tilde{c}(v,\eta)^{-1} - \epsilon} + \dfrac{MK}{\tilde{c}(v,\eta)^{-1} - \epsilon} \, \log\left(\left(\dfrac{2 \, f^{-1}(\delta)}{\tilde{c}(v,\eta)^{-1} - \epsilon}\right)^2 + \dfrac{2 \, f^{-1}(\delta)}{\tilde{c}(v,\eta)^{-1} - \epsilon}  \right), \\
        \\
        \bar{t} \ge  \max \bigg\{\dfrac{3}{\tilde{c}(v,\eta)^{-1} - \epsilon}, \ 1 \bigg\}, \\
        \\
        \bar{t} \le \dfrac{2 \, f^{-1}(\delta)}{\tilde{c}(v,\eta)^{-1} - \epsilon}. \\
        \end{cases}
    \label{eq:sufficiency4}
\end{align}
For brevity in notation, let the right-hand sides of the first, second, and third lines in~\eqref{eq:sufficiency4} be denoted respectively as $a(\delta)$, $b(\delta)$, and $c(\delta)$. Using the fact that $\lim_{\delta \downarrow 0} f^{-1}(\delta) = +\infty$, we get that 
$$
    \lim_{\delta \downarrow 0} (c(\delta) - a(\delta)) = +\infty, \quad 
    \lim_{\delta \downarrow 0} (a(\delta)- b(\delta)) = +\infty.
$$ 
Let $T_{{\rm thres}}(v,\eta,\epsilon,\delta)$ be as defined in the statement of the lemma. The above limiting relations imply that there exists $\delta_{\rm thres}(v,\eta,\epsilon) >0$ sufficiently small such that $H(T_{{\rm thres}}(v,\eta,\epsilon,\delta), \delta) = 1$ for all $\delta\in (0,\delta_{\rm thres}(v,\eta,\epsilon))$. This completes the desired proof.
\end{proof}

\begin{lemma}
\label{lemma:zt_t_to_c_tilde}
Fix $v \in \mathcal{P}$, $\eta>0$, and consider the non-stopping version of $\textsc{MO-BAI}$ (one in which the stopping rule corresponding to Lines 11-14 in Algorithm~\ref{alg:mo-bai} are not executed). Under the aforementioned policy and under the instance $v$,
\begin{equation}
    \lim_{t \rightarrow \infty} \frac{Z(t)}{t} = \tilde{c}(v,\eta)^{-1}
    \quad \text{almost surely}.
\end{equation}
\end{lemma}
\begin{proof}[Proof of Lemma~\ref{lemma:zt_t_to_c_tilde}]
Let $\widehat{N}_{\cdot,t} = [\widehat{N}_{i,t}: i \in [K]]^\top \in \Gamma$ be defined as 
\begin{equation}
    \widehat{N}_{i,t} \coloneqq \frac{N_{i,t}}{ t}, \quad i \in [K]
    \label{eq:hat-N-im-t}.
\end{equation}
In \eqref{eq:hat-N-im-t}, $N_{i,t}$ is the total number of times arm $i$ is pulled up to time $t$. Then, from \eqref{eq:test-statistic}, we have
\begin{align}
    \frac{Z(t)}{t} 
    & = \frac{1}{t} \, \min_{m\in [M]} \ \min_{i\in[K] \setminus \widehat{i}_m(t) } \frac{N_{i,t} \,  N_{\widehat{i}_m(t),t} \, \widehat{\Delta}^2_{i,m}(t)}{2(N_{i,t}+ N_{\widehat{i}_m,t})} \nonumber \\
    & = \min_{m\in [M]} \ \min_{i\in[K] \setminus \widehat{i}_m(t) } \frac{\widehat{N}_{i,t} \,  \widehat{N}_{\widehat{i}_m(t),t} \, \widehat{\Delta}^2_{i,m}(t)}{2(\widehat{N}_{i,t}+ \widehat{N}_{\widehat{i}_m,t})}.
    \label{eq:proof-of-lim-zt-t-1}
\end{align}
In addition, by Lemma~\ref{lemma:g_to_c_star}, we have 
\begin{equation}
\label{eq:g_hat_to_c}
   \lim_{t \rightarrow +\infty} g_v(\widehat{\omega}_{\cdot,t}) = \tilde{c}(v,\eta)^{-1} = \sup_{\omega \in \Gamma^{(\eta)}} g_v(\omega) \quad \text{almost surely.}
\end{equation}
Also, noting that $g_v(\cdot)$ is a continuous function on $\Gamma$ with respect to the Euclidean norm $\| \cdot \|_2$, and that that $\Gamma$ is compact, the Heine–Cantor theorem implies that we get that $g_v(\cdot)$ is uniformly continuous on $\Gamma$. Using this fact and taking limits as $t \to +\infty$ in \eqref{eq:proof-of-lim-zt-t-1}, we get
\begin{align}
    \lim_{t\rightarrow +\infty} \frac{Z(t)}{t} 
    &=  \lim_{t \rightarrow +\infty} \ \min_{m\in [M]}  \ \min_{i\in[K] \setminus \widehat{i}_m(t) } \frac{\widehat{N}_{i,t} \,  \widehat{N}_{\widehat{i}_m(t),t} \, \widehat{\Delta}^2_{i,m}(t)}{2(\widehat{N}_{i,t}+ \widehat{N}_{\widehat{i}_m,t})} \nonumber \\
    & \stackrel{(a)}{=} \lim_{t \rightarrow +\infty} \min_{m\in [M]} \min_{i\in[K] \setminus i^*_m(v) } \frac{\widehat{N}_{i,t} \, \widehat{N}_{\widehat{i}_m(t),t} \, {\Delta}^2_{i,m}(v)}{2(\widehat{N}_{i,t}+ \widehat{N}_{\widehat{i}_m,t})} \nonumber  \\
    & \stackrel{(b)}{=} \lim_{t \rightarrow +\infty} g_v(\widehat{N}_{\cdot,t}) \nonumber  \\
    & \stackrel{(c)}{=} \lim_{t \rightarrow +\infty} g_v(\widehat{\omega}_{\cdot,t}) \nonumber   \\
    & \stackrel{(d)}{=}   \tilde{c}(v,\eta)^{-1},
\end{align}
where $(a)$ above follows from the strong law of large numbers, $(b)$ follows from the definition of $g_v(\cdot)$, $(c)$ follows from the fact that $\lim_{t \rightarrow \infty} \| \widehat{N}_{\cdot,t} - \widehat{\omega}_{\cdot,t}  \|_2 =0 $ and that $g_v(\cdot)$ is uniformly continuous on $\Gamma$, and $(d) $ follows from~\eqref{eq:g_hat_to_c}.
\end{proof}


\begin{lemma}
\label{lemma:finite_stopping}
Fix $v \in \mathcal{P}$, $\delta \in (0,1)$, and $\eta>0$. Under $\textsc{MO-BAI}$,
\begin{equation}
    \tau_\delta < +\infty \quad \text{almost surely}.
    \label{eq:stopping-time-finite-almost-surely}
\end{equation}
\end{lemma}
\begin{proof}[Proof of Lemma~\ref{lemma:finite_stopping}]
Note that we have almost surely
\begin{align}
   \lim_{t \rightarrow \infty} \frac{\beta(t, \delta)}{Z(t)}
   & \stackrel{(a)}{=}  \lim_{t \rightarrow \infty} \frac{MK\log(t^2+t)+ f^{-1}(\delta)}{t \, \dfrac{Z(t)}{t}} \nonumber \\
   & \stackrel{(b)}{=}  \lim_{t \rightarrow \infty} \frac{MK\log(t^2+t)+ f^{-1}(\delta)}{t } \cdot \frac{1}{\tilde{c}(v,\eta)^{-1}} \nonumber \\
   & =  0,
\end{align}
where $(a)$ follows from the definitions of $Z(t)$ and $\beta(t, \delta)$, and $(b)$ follows from Lemma~\ref{lemma:zt_t_to_c_tilde}.
Therefore, there exists a random variable $T'$ such that $0<T'<+\infty$ almost surely, and $Z(t) > \beta(t, \delta)$ for all $t \ge T'$, thereby proving that $\tau_\delta$ is finite almost surely.
\end{proof}


\begin{lemma}
\label{lemma:relax_c}
Fix $v \in \mathcal{P}$. For any $\eta>0$,
\begin{equation}
    c^*(v) \le (1+\eta) \, \tilde{c}(v,\eta).
    \label{eq:relation-between-c-star-v-and-c-star-v-eta}
\end{equation}
\end{lemma}
\begin{proof}[Proof of Lemma~\ref{lemma:relax_c}]
Fix an arbitrary $\eta>0$. Recall that   
\begin{align}  
    c^*(v)^{-1} 
    &= \sup_{\omega \in \Gamma} \ \min_{m\in [M]} \ \min_{i\in[K] \setminus i^*_m(v) } \  \frac{\omega_i \, \omega_{i^*_m(v)} \, \Delta^2_{i,m}(v)}{2(\omega_i+\omega_{i^*_m(v)})}, \\
    \tilde{c}(v,\eta)^{-1} 
    &= \sup_{\omega \in \Gamma^{(\eta)}} \ \min_{m\in [M]} \ \min_{i\in[K] \setminus i^*_m(v)} \ \frac{\omega_i \,  \omega_{i^*_m(v)} \, \Delta^2_{i,m}(v)}{2(\omega_i+\omega_{i^*_m(v)})}.
\end{align}
Let
\begin{equation}
    \omega^*(v) \in \arg\sup_{\omega \in \Gamma} \ \min_{m\in [M]} \ \min_{i\in[K] \setminus i^*_m(v)} \frac{\omega_i \, \omega_{i^*_m(v)} \, \Delta^2_{i,m}(v)}{2(\omega_i+\omega_{i^*_m(v)})}
\end{equation}
be chosen arbitrarily, and let $\mathbf{\omega'}\in \mathbb{R}^K$ be defined as
$$
    \omega'_i \coloneqq \frac{\omega_i^*(v)}{1+\eta} + \frac{\eta}{(1+\eta) \, K}, \quad i \in [K].
$$
Note that
\begin{align}
    \sum_{i=1}^K \omega'_i 
    & = \sum_{i=1}^K \frac{\omega_i^*(v)}{1+\eta} + \frac{\eta}{(1+\eta)K} \nonumber \\
    & = \frac{1}{1+\eta} + \frac{\eta}{1+\eta} \nonumber \\
    & = 1, 
    \label{eq:in_gamma_eta_1}
\end{align}
and for all $i \in [K]$,
\begin{equation}
\label{eq:in_gamma_eta_2}
\omega'_i  \ge  \frac{\eta}{(1+\eta)K}.   
\end{equation}
Then,~\eqref{eq:in_gamma_eta_1} and \eqref{eq:in_gamma_eta_2} together imply that $\omega' \in \Gamma^{(\eta)}$, as a result of which we have
\begin{align}
    \tilde{c}(v,\eta)^{-1} &= \sup_{\omega \in \Gamma^{(\eta)}}\ \min_{m\in [M]} \ \min_{i\in[K] \setminus i^*_m(v)}  \ \frac{\omega_i \, \omega_{i^*_m(v)} \, \Delta^2_{i,m}(v)}{2(\omega_i+\omega_{i^*_m(v)})} \nonumber \\
   & \stackrel{(a)}{\ge} \min_{m\in [M]} \ \min_{i\in[K] \setminus i^*_m(v)} \ \frac{\omega'_i \,  \omega'_{i^*_m(v)} \, \Delta^2_{i,m}(v)}{2(\omega'_i+\omega'_{i^*_m(v)})} \nonumber \\
   &\stackrel{(b)}{\ge} \min_{m\in [M]} \ \min_{i\in[K] \setminus i^*_m(v)} \  \frac{\left(\omega^*_i(v)/(1+\eta)\right) \, \left(\omega_{i^*_m(v)}/(1+\eta)\right) \, \Delta^2_{i,m}(v)}{2\left(\omega'_i/(1+\eta)+\omega_{i^*_m(v)}/(1+\eta)\right)} \nonumber \\
   & \stackrel{(c)}{=} \frac{c^*(v)^{-1}}{1+\eta},
\end{align}
where $(a)$ follows from the fact that $\omega' \in \Gamma^{(\eta)}$, $(b)$ follows from the fact that $\omega'_i \ge \frac{\omega_i^*(v)}{1+\eta}$ for all $i\in[K]$, and $(c)$ follows from the definition of $c^*(v)^{-1}$ in \eqref{eq:cstar}.
\end{proof}


Given any $\epsilon>0$, let $T_{\rm gap}(v,\eta,\epsilon)$ denote the smallest positive integer-valued random variable such that
\begin{equation}
     \left\lvert \frac{Z(t)}{t} -\tilde{c}(v,\eta)^{-1} \right\rvert \le \epsilon \quad \forall \ t \ge T_{\rm gap}(v,\eta,\epsilon).
     \label{eq:T-cvg-definition}
\end{equation}
\begin{lemma}
\label{lemma:T_gap_finite}
Fix instance $v \in \mathcal{P}$ with arm means $\{\mu_{i,m}: i \in [K], m \in [M]\}$ and $\eta > 0$. For every $\epsilon >0$,
\begin{equation}
    \mathbb{E}_v^{\textsc{MO-BAI}} [T_{\rm gap}(v,\eta,\epsilon)] < +\infty.
    \label{eq:expected-value-of-T-gap-finite}
\end{equation}
\end{lemma}
\begin{proof}[Proof of Lemma~\ref{lemma:T_gap_finite}]
Fix $\epsilon>0$. Recall the quantity $\widehat{N}_{i,t}=\frac{{N}_{i,t}}{t}$. For any $\xi >0$, let $T_{\widehat{N}}(\xi)$ denote the smallest positive integer such that
\begin{equation}
     \max_{i \in [K]} \ \left\lvert {\widehat{N}_{i,t}}-\widehat{\omega}_{i,t} \right\rvert \le \xi \quad \forall t>T_{\widehat{N}}(\xi),
     \label{eq:T_hat_N}
\end{equation}
let $T_{\mu}(\xi)$ denote the smallest positive integer such that
\begin{equation}
     \max_{(i,m) \in [K] \times [M]} \ \left\lvert \mu_{i,m}(v)-\widehat{\mu}_{i,m}(l_t) \right\rvert \le \xi \quad \forall t>T_{\mu}(\xi),
     \label{eq:T_hat_mu}
\end{equation}
and let $T_{\widehat{\omega}}(\xi)$ denote the smallest positive integer such that
\begin{equation}
    \max \left\lbrace \left|\frac{Z(t)}{t}- g_{v}(\widehat{N}_{\cdot,t}) \right|, \quad  \left|\tilde{c}(v,\eta)^{-1}- g_{v}(\widehat{\omega}_{\cdot,t}) \right| \right\rbrace < \xi \quad \forall t>T_{\widehat{\omega}}(\xi).
    \label{eq:T-hat-omega}
\end{equation}
Recall from the proof of Lemma~\ref{lemma:zt_t_to_c_tilde} that $g_v(\cdot)$ is uniformly continuous on $\Gamma$.
Then, from~\eqref{eq:T_hat_N} and~\eqref{eq:T-hat-omega}, we get that there exists $\epsilon_1>0$ such that
\begin{equation}
    T_{\rm gap}(v,\eta,\epsilon) < \max\{T_{\widehat{N}}(\epsilon_1), T_{\widehat{\omega}}(\epsilon_1) \} \quad \text{almost surely}.
    \label{eq:T_gap_less_1}
\end{equation}
Replacing $t_2$ with $t$ in \eqref{eq:proof-of-upper-bound-1}, we note that almost surely, 
\begin{align}
    &\left | \tilde{c}(v,\eta)^{-1}-g_v(\widehat{\omega}_{\cdot,t}) \right |  \nonumber\\
    &\le  \frac{2}{\sqrt{t}} \, \tilde{c}(v,\eta)^{-1}  + 11 \epsilon_{\lceil \sqrt{t} \rceil}(v)+  \frac{2 \, \log(t) \, \overline{C}_{\lceil \sqrt{t} \rceil}(\eta)}{t} \nonumber\\
    &= \frac{2}{\sqrt{t}}\tilde{c}(v,\eta)^{-1} + 11 \sup_{t' \ge l_{\lceil \sqrt{t} \rceil}} \left( \sup_{\omega \in \Gamma^{(\eta)}} \big\lvert g_v(\omega)-g_{\widehat{v}_{t'}}(\omega) \big\rvert  \right) + \frac{2\log(t)}{t} \sup_{t' \ge l_{\lceil \sqrt{t} \rceil}} C(\widehat{v}_{t'},\eta) \nonumber\\
    &\le  \frac{2}{\sqrt{t}} \, \tilde{c}(v,\eta)^{-1}  + 11 \sup_{t' \ge \lceil \sqrt{t} \rceil/2} \left( \sup_{\omega \in \Gamma^{(\eta)}} \big\lvert g_v(\omega)-g_{\widehat{v}_{t'}}(\omega) \big\rvert  \right) +  \frac{2\log(t)}{t} \sup_{t' \ge \lceil \sqrt{t} \rceil/2} C(\widehat{v}_{t'},\eta),
    \label{eq:finite_t_gap_c0}
\end{align}
where the last line above follows by noting that $l_t \geq t/2$ for any $t$ (see the definition of $l_t$ in Line 8 of Algorithm~\ref{alg:mo-bai}). Also, \eqref{eq:g_omega_gap} implies that almost surely, 
\begin{equation}
    \sup_{\omega \in \Gamma^{(\eta)}} \lvert g_v (\omega) - g_{\widehat{v}_{l_t}}(\omega) \rvert \le \max_{(i,m) \in [K]\times [M]} \lvert \Delta_{i,m}(\widehat{v}_{l_t}) - \Delta_{i,m}(v) \rvert \quad \forall t \ge \max_{m \in [M]} T_m,
    \label{eq:finite_t_gap_c1}
\end{equation}
where $T_m$ (for any $m \in [M]$) is as defined in~\eqref{eq:def_t_m}. Notice that for each $m \in [M]$, we have $i_m^*(\widehat{v}_{l_t}) = i_m^*(v)$ for all $t \geq T_m$ almost surely by the definition of $T_m$. Therefore, it follows that almost surely,
\begin{align}
    &\max_{(i,m) \in [K]\times [M]} \lvert \Delta_{i,m}(\widehat{v}_{l_t}) - \Delta_{i,m}(v) \rvert \nonumber\\*
    &= \max_{(i,m) \in [K]\times [M]} \bigg\lvert \mu_{i_m^*(v), m} - \mu_{i,m}(v) - \big(\widehat{\mu}_{i_m^*(\widehat{v}_{l_t}),m}(l_t) - \widehat{\mu}_{i,m}(l_t)\big) \bigg\rvert \nonumber\\
    &= \max_{(i,m) \in [K]\times [M]} \bigg\lvert \mu_{i_m^*(v), m} - \widehat{\mu}_{i_m^*(v),m}(l_t) - \big(\mu_{i,m}(v)  - \widehat{\mu}_{i,m}(l_t)\big) \bigg\rvert \nonumber\\
    &\le 2 \, \max_{(i,m) \in [K]\times [M]} |\mu_{i,m}(v) - \widehat{\mu}_{i,m}(l_t)|
    \label{eq:proof-of-upper-bound-2}
\end{align}
for all $t \geq \max_{m \in [M]} T_m$. Hence, we get that there exists $\epsilon_2>0$ such that
\begin{equation}
    \max_{m \in [M]} T_m < T_{\mu}(\epsilon_2) \quad \text{almost surely}.
    \label{eq:finite_t_gap_c2}
\end{equation}
In addition, by Lemma~\ref{lemma:finite_cuvature}, we have $\forall t'>0$,
\begin{equation}
    C(\widehat{v}_{t'},\eta) \le \max_{i \neq i^*_m(\widehat{v}_{t'})} \frac{2\Delta_{i,m}^2(\widehat{v}_{t'})K(1+\eta)}{\eta} \quad \text{almost surely}.    
    \label{eq:finite_t_gap_c3}
\end{equation}
Then, combining~\eqref{eq:finite_t_gap_c0}, ~\eqref{eq:finite_t_gap_c1}, \eqref{eq:finite_t_gap_c2} and~\eqref{eq:finite_t_gap_c3}, we get that there exists $\epsilon_3\in(0,\epsilon_2)$ and $T'\in \mathbb{N}$ such that
\begin{equation}
    T_{\widehat{\omega}}(\epsilon_1) < \max\{T', T_\mu(\epsilon_3) \} \quad \text{almost surely}.
    \label{eq:T_gap_less_2}
\end{equation}
Combining~\eqref{eq:T_gap_less_1} and~\eqref{eq:T_gap_less_2}, and using the fact that $\max\{a,b\} \leq a+b$ for all $a,b \geq 0$, we have
\begin{equation}
\label{eq:T_gap_finite_final_0}
     T_{\rm gap}(v,\eta,\epsilon) \leq T_{\widehat{N}}(\epsilon_1)+T_{\mu}(\epsilon_3)+T' \quad \text{almost surely}.
\end{equation}
Notice that
\begin{align}
    \big
    |\widehat{N}_{i,t} - \widehat{\omega}_{i,t}\big|
    = \bigg|\frac{B_{i,t}}{t}\bigg| \leq \frac{1}{t},
\end{align}
where the equality follows from the definition of $\widehat{\omega}_{i,t}$ in Line 7 of Algorithm~\ref{alg:mo-bai}, and the inequality follows by noting that the buffer size $|B_{i,t}| \leq 1$ for all $t$. Therefore, we have 
\begin{equation}
    T_{\widehat{N}}(\epsilon_1) < \frac{1}{\epsilon_1} \quad \text{almost surely}.
    \label{eq:T_gap_finite_final_1}
\end{equation}

For any $i\in[K]$,  $m\in[M]$, $t\in\mathbb{N}$, and $\xi>0$, let
$$
    \mathcal{E}(t,i,m,\xi) \coloneqq \big\{  \lvert \widehat{\mu}_{i,m}(t)-\mu_{i,m}(v) \rvert > \xi \big\}.
$$
Then, we have
\begin{align}
    \mathbb{E}_v^{\textsc{MO-BAI}}\left[T_\mu(\epsilon_3) \right]
    &= \sum_{t=1}^{\infty} \mathbb{P}_v^{\textsc{MO-BAI}}\left(T_\mu(\epsilon_3) >t \right) \nonumber\\
    &\leq K + \sum_{t=K+1}^{\infty} \mathbb{P}_v^{\textsc{MO-BAI}} \left(T_\mu(\epsilon_3) >t \right) \nonumber \\
    &\leq K + \sum_{i=1}^K  \ \sum_{m=1}^M \ \sum_{t=K+1}^{\infty} \mathbb{P}_v^{\textsc{MO-BAI}}\left( \bigcup_{t'>l_t} \mathcal{E}(t',i,m,\epsilon_3) \right) \nonumber \\
    &\stackrel{(a)}{\leq} K + \sum_{i=1}^K \  \sum_{m=1}^M \ \sum_{t=K+1}^{\infty} \ \sum_{t'= t}^{\infty} \ \mathbb{P}_v^{\textsc{MO-BAI}}\left(\mathcal{E}(t',i,m,\epsilon_3) \right), 
    \label{eq:T_gap_finite_final_2}
\end{align}
where $(a)$ above follows from the union bound. We now observe that for any $t' > K$,
\begin{align}
    \mathcal{E}(t', i, m, \xi)
    &= \bigg\{\bigg|\widehat{\mu}_{i,m}(t') - \mu_{i,m}(v)\bigg| > \xi\bigg\} \nonumber\\
    &= \left\{\dfrac{\bigg|\sum_{s=1}^{t'} \mathbf{1}_{\{A_s=i\}} \, (X_{A_s,m}(s) - \mu_{i,m}(v))\bigg|}{N_{i,t'}} > \xi\right\} \nonumber\\
    &\subseteq \left\{\bigg|\sum_{s=1}^{t'} \mathbf{1}_{\{A_s=i\}} \, (X_{A_s,m}(s) - \mu_{i,m}(v))\bigg| > \xi \, 
    \left(\frac{\eta}{K  (1+\eta)}t'-1 \right)\right\}
    \label{eq:T_gap_finite_final_3},
\end{align}
where the last line above follows by noting that for all $t'$, under $\textsc{MO-BAI}$,
\begin{align}
    N_{i,t'} 
    &= t' \, \widehat{\omega}_{i, t'} - B_{i,t'} \nonumber\\
    &\geq \frac{\eta}{K  (1+\eta)}t'- 1.
\end{align}
The inequality above follows by observing that $\widehat{\omega}_{\cdot, t'} \in \Gamma^{(\eta)}$ for all $t'$. We note that 
$$
    \bigg\{\sum_{s=1}^{t'} \mathbf{1}_{\{A_s=i\}} \, (X_{A_s,m}(s) - \mu_{i,m}(v))\bigg\}_{t'=K+1}^{\infty}
$$ 
is a bounded martingale difference sequence with finite variance.
Using \eqref{eq:T_gap_finite_final_3} in \eqref{eq:T_gap_finite_final_2} together with martingale concentration bounds \citep[Theorem 1.2A]{delapena1999general}, we get
\begin{align}
    \mathbb{E}_v^{\textsc{MO-BAI}}\left[T_\mu(\epsilon_3) \right]
    &\le K + \sum_{i=1}^K \  \sum_{m=1}^M \ \sum_{t=K+1}^{\infty} \ \sum_{t'= t}^{\infty} \ \exp\left(-\Big(\frac{\eta \, t'}{K  (1+\eta)}-1\Big) c_{\epsilon_3} \right) \nonumber \\
    &= K + \sum_{i=1}^K \ \sum_{m=1}^M \ \sum_{t'= K+1}^{\infty} \ t' \, \exp\left(-\Big(\frac{\eta \, t'}{K  (1+\eta)}-1\Big) c_{\epsilon_3} \right) \nonumber \\
    & < +\infty.
    \label{eq:T_gap_finite_final_3'}
\end{align}
In the above set of relations, $c_{\epsilon_3}$ is a positive constant that depends only on $\epsilon_3$. Finally, combining~\eqref{eq:T_gap_finite_final_0},~\eqref{eq:T_gap_finite_final_1}, and~\eqref{eq:T_gap_finite_final_3'}, we arrive at the desired result. This completes the proof.
\end{proof}

With the ingredient of above lemmas, we are ready to prove Theorem~\ref{thm:upperbound}.
\begin{proof}[Proof of Theorem~\ref{thm:upperbound}]
Fix $\eta>0$ and instance $v\in \mathcal{P}$. Consider $\textsc{MO-BAI}$. 
By Lemma~\ref{lemma:zt_t_to_c_tilde}, we have $T_{\rm gap}(v,\eta,\epsilon) < +\infty$ almost surely. For any $\epsilon \in \left(0, \tilde{c}(v,\eta)^{-1}\right)$ and $\delta \in (0, \delta_{\rm thres}(v,\eta,\epsilon))$, it follows from Lemma~\ref{lemma:Tlast} that
\begin{equation*}
    Z(t) > \beta(t, \delta) \quad \forall\, t \ge \max\big\{T_{\rm gap}(v,\eta,\epsilon),T_{\rm thres}(v,\eta,\delta, \epsilon), K\big\} \quad  \text{almost surely}, 
\end{equation*}
which implies that almost surely,
\begin{align}
 \tau_\delta & \le  \max\big\{T_{\rm gap}(v,\eta,\epsilon),T_{\rm thres}(v,\eta,\delta, \epsilon), K\big\}  +1 \nonumber  \\
 & \le T_{\rm gap}(v,\eta,\epsilon)+T_{\rm thres}(v,\eta,\delta, \epsilon)+ K +1.
\label{eq:real_stop_time}
\end{align}
Hence, for any $\epsilon \in \left(0, \tilde{c}(v,\eta)^{-1} \right)$, we have
almost surely
\begin{align}
    & \limsup_{\delta \downarrow 0} \frac{\tau_\delta}{\log\left(\frac{1}{\delta}\right)} \nonumber \\
    & \stackrel{(a)}{\le} \limsup_{\delta \downarrow 0} \  \frac{ T_{\rm gap}(v,\eta,\epsilon)+T_{\rm thres}(v,\eta,\delta, \epsilon)+ K +1 }{\log\left(\frac{1}{\delta}\right)} \nonumber \\
    &\stackrel{(b)}{=} \limsup_{\delta \downarrow 0} \ \frac{ T_{\rm thres}(v,\eta,\delta, \epsilon) }{\log\left(\frac{1}{\delta}\right)} \nonumber \\
    &\stackrel{(c)}{=} \limsup_{\delta \downarrow 0} \ \dfrac{ \dfrac{f^{-1}(\delta)}{\tilde{c}(v,\eta)^{-1} - \epsilon} + \dfrac{MK}{\tilde{c}(v,\eta)^{-1} - \epsilon} \, \log\left(\left(\dfrac{2 \, f^{-1}(\delta)}{\tilde{c}(v,\eta)^{-1} - \epsilon}\right)^2 + \dfrac{2 \, f^{-1}(\delta)}{\tilde{c}(v,\eta)^{-1} - \epsilon}  \right) +1 }{\log\left(\frac{1}{\delta}\right)} \nonumber \\
    & \stackrel{(d)}{=} \frac{1}{\tilde{c}(v,\eta)^{-1} - \epsilon} \nonumber \\
    & \stackrel{(e)}{\le} \frac{1}{((1+\eta) \, c^*(v))^{-1} - \epsilon},
    \label{eq:resultwithepsilon}
\end{align}
where $(a)$ follows from~\eqref{eq:real_stop_time}, $(b)$ follows from the fact that $T_{\rm gap}(v,\eta,\epsilon) < +\infty$ almost surely and that $T_{\rm gap}(v,\eta,\epsilon)$ does not depend on $\delta$, $(c)$ follows from the definition of $T_{\rm thres}(\cdot)$, $(d)$ follows from Lemma~\ref{lemma:f_delta}, and $(e)$ follows from Lemma~\ref{lemma:relax_c}. Letting $\epsilon \rightarrow 0$, we have
$$
\limsup_{\delta \downarrow 0} \frac{\tau_\delta}{\log\left(\frac{1}{\delta}\right)} \le {(1+\eta) \, c^*(v)} \quad \text{almost surely}.
$$
In addition, we have
\begin{align}
    & \limsup_{\delta \downarrow 0}  \ \frac{\mathbb{E}_v^{\textsc{MO-BAI}}[\tau_\delta]}{\log\left(\frac{1}{\delta}\right)} \nonumber \\
    & \stackrel{(a)}{\le} \limsup_{\delta \downarrow 0} \  \frac{\mathbb{E}_v^{\textsc{MO-BAI}}[T_{\rm gap}(v,\eta,\epsilon)] + T_{\rm thres}(v,\eta,\delta, \epsilon) + K + 1 }{\log\left(\frac{1}{\delta}\right)} \nonumber \\
    &\stackrel{(b)}{=} \limsup_{\delta \downarrow 0} \ \frac{ T_{\rm thres}(v,\eta,\delta, \epsilon) }{\log\left(\frac{1}{\delta}\right)} \nonumber \\
    &\stackrel{(c)}{=} \limsup_{\delta \downarrow 0} \dfrac{ \dfrac{f^{-1}(\delta)}{\tilde{c}(v,\eta)^{-1} - \epsilon} + \dfrac{MK}{\tilde{c}(v,\eta)^{-1} - \epsilon} \, \log\left(\left(\dfrac{2 \, f^{-1}(\delta)}{\tilde{c}(v,\eta)^{-1} - \epsilon}\right)^2 + \dfrac{2 \, f^{-1}(\delta)}{\tilde{c}(v,\eta)^{-1} - \epsilon}  \right) +1 }{\log\left(\frac{1}{\delta}\right)} \nonumber \\
    & \stackrel{(d)}{=} \frac{1}{\tilde{c}(v,\eta)^{-1} - \epsilon} \nonumber \\
    & \stackrel{(e)}{\le} \frac{1}{((1+\eta) \, c^*(v))^{-1} - \epsilon},
    \label{eq:resultwithepsilon-1}
\end{align}
where $(a)$ follows from~\eqref{eq:real_stop_time}, $(b)$ follows from the fact that $\mathbb{E}_v^{\textsc{MO-BAI}} [T_{\rm gap}(v,\eta,\epsilon)] < +\infty$ (see Lemma~\ref{lemma:T_gap_finite}) and that $\mathbb{E}_v^{\textsc{MO-BAI}}(T_{\rm gap}(v,\eta,\epsilon))$ does not depend on $\delta$, $(c)$ follows from the definition of $T_{\rm thres}(\cdot)$, $(d)$ follows from Lemma~\ref{lemma:f_delta}, and $(e)$ follows from Lemma~\ref{lemma:relax_c}. Letting $\epsilon \rightarrow 0$, we get
$$
\limsup_{\delta \downarrow 0} \frac{\mathbb{E}_v^{\textsc{MO-BAI}}[\tau_\delta]}{\log\left(\frac{1}{\delta}\right)} \le {(1+\eta) \, c^*(v)},
$$
thereby arriving at the desired result. This completes the proof.
\end{proof}

\end{document}